\documentclass[10pt]{article} % For LaTeX2e
\usepackage[preprint]{tmlr}
\usepackage{amsmath,amsfonts,bm, amsthm}
\usepackage{float}
\usepackage{url}
\usepackage[ruled,vlined]{algorithm2e}
\usepackage{xcolor}

\usepackage{graphicx}
\usepackage{subfigure}
\usepackage{tikz}
\usepackage{tikz-cd}
\newtheorem{proposition}{Proposition}
\newtheorem{theorem}{Theorem}
\newtheorem{lemma}{Lemma}
\newtheorem{corollary}{Corollary}
\newtheorem{definition}{Definition}

\def\C{\ensuremath{\mathcal{C}}}
\def\Ac{\ensuremath{\mathcal{A}}}
\def\F{\ensuremath{\mathcal{F}}}
\def\sP{{\mathbb{P}}}
\def\mB{{\bm{B}}}
\def\mP{{\bm{P}}}
\def\mT{{\bm{T}}}
\def\mC{{\bm{C}}}
\def\mTg{{\bm{T}}_\gamma}
\def\vzero{{\bm{0}}}
\def\vc{{\bm{c}}}
\def\ve{{\bm{e}}}
\def\vu{{\bm{u}}}
\def\vx{{\bm{x}}}
\def\vy{{\bm{y}}}
\def\vw{{\bm{w}}}
\def\vz{{\bm{z}}}
\def\vV{{\bm{V}}}
\def\vmu{{\bm{\mu}}}
\def\vnu{{\bm{\nu}}}
\def\evc{{c}}
\def\eve{{e}}
\def\evu{{u}}
\def\evx{{x}}
\def\evy{{y}}
\def\evw{{w}}
\def\evV{{V}}
\def\evmu{{\mu}}
\def\sB{{\mathbb{B}}}

\newcommand{\E}{\mathbb{E}}
\newcommand{\R}{\mathbb{R}}
\newcommand{\St}{\mathcal{S}}

\usepackage{float}
\usepackage{array}
\newcolumntype{M}[1]{>{\centering\arraybackslash}m{#1}}

\DeclareMathOperator*{\argmin}{arg\,min}
\DeclareMathOperator*{\argmax}{arg\,max}

\definecolor{OliveGreen}{HTML}{3C8031}
\definecolor{Fuchsia}{HTML}{8C368C}

\title{Apprenticeship learning with prior beliefs using inverse optimization}

\author{ \name Mauricio Junca \email mj.junca20@uniandes.edu.co\\
 \addr Department of Mathematics\\
 Universidad de los Andes, Bogot\'{a}, Colombia.
 \AND
 \name Esteban Leiva \email leivamon@usc.edu\\
 \addr Department of Industrial and Systems Engineering\\
 University of Southern California, Los Angeles, CA, US.}

\def\month{02}  % Insert correct month for camera-ready version
\def\year{2025} % Insert correct year for camera-ready version
\def\openreview{\url{https://openreview.net/forum?id=XXXX}} % Insert correct link to OpenReview for camera-ready version

\begin{document}

\maketitle

\begin{abstract}
 The relationship between inverse reinforcement learning (IRL) and inverse optimization (IO) for Markov decision processes (MDPs) has been relatively underexplored in the literature, despite addressing the same problem. In this work, we revisit the relationship between the IO framework for MDPs, IRL, and apprenticeship learning (AL). We incorporate prior beliefs on the structure of the cost function into the IRL and AL problems, and demonstrate that the convex-analytic view of the AL formalism emerges as a relaxation of our framework. Notably, the AL formalism is a special case in our framework when the regularization term is absent. Focusing on the suboptimal expert setting, we formulate the AL problem as a regularized min-max problem. The regularizer plays a key role in addressing the ill-posedness of IRL by guiding the search for plausible cost functions. To solve the resulting regularized-convex-concave-min-max problem, we use stochastic mirror descent (SMD) and establish convergence bounds for the proposed method. Numerical experiments highlight the critical role of regularization in learning cost vectors and apprentice policies.
\end{abstract}

\section{Introduction} \label{sec:intro}

In scenarios where an agent must learn to navigate in a random or uncertain environment, it is a common practice to model the situation as a Markov decision process (MDP) and apply reinforcement learning (RL). The goal in RL is to find a policy that minimizes the total expected discounted cost for the agent. Usually, it is assumed that the cost function is known; however, specifying this function is difficult for most real-life scenarios \citep{Ng2000}. Moreover, an incorrect specification of the cost function can lead to unintended and potentially detrimental effects on the agent's behavior \citep{Amodei2016, Hadfieldmenell2020}. Consider the problem of driving: should the agent be rewarded for arriving quickly, safely, or cheaply, and how should the importance of each factor be balanced? \par

Inverse reinforcement learning (IRL) tackles this problem by reducing the work of manually designing the cost function and using observations of an expert agent's actions. Specifically, IRL aims to infer the cost function that the expert is optimizing based on recorded behavior and a model of the environment. Returning to the driving example, this approach involves observing an expert driver’s behavior and deducing the underlying objective that guides their decisions. However, the goal extends beyond identifying the cost function; in many cases, there is a desire to emulate the expert's actions, much like a student assimilating knowledge from a mentor. For instance, when children learn to run, they are not explicitly given a cost function to optimize, but an expert shows them demonstrations of how they should run. Building on this idea, learning from demonstrations (LfD) and imitation learning (IL) seek to derive a policy that matches or surpasses the expert's performance.\par

IRL was first informally proposed by  \cite{Russel1998}, and \cite{Ng2000} introduced algorithms for three scenarios: (1) when the policy, transition dynamics, and a finite state space are known; (2) when the state space is infinite; and (3) when the policy is unknown, but sample trajectories are available. Several methods have since been proposed, including a maximum margin approach \citep{Ratliff2006}, Bayesian frameworks \citep{Ramachandran2007}, and maximum entropy techniques \citep{Ziebart2008}. However, all of these methods rely on RL as a subroutine within an inner loop, leading to significant computational expenses. \par

Furthermore, the IRL problem is ill-posed \citep{Ng2000}, as multiple cost functions can explain an agent's behavior. This challenge has garnered increasing attention in the literature, with several works focusing on identifying the set of feasible cost functions that account for the expert's behavior \citep{Metelli2021, Lindner2022}. For our analysis, it is important to note that the work of \cite{Metelli2021, Metelli2023, Lazzati2025} assumes access to a generative-model oracle for both the transition dynamics of the MDP and the expert's policy. A key finding in this body of work is that learning the feasible cost set is inefficient and infeasible for large state spaces, as the sample complexity scales heavily with the size of the state space. \par

In the context of LfD or IL, the literature often adopts the apprenticeship learning (AL) formalism proposed by \cite{Abbeel2004}, which assumes access to a set of expert demonstrations and that the unknown true cost function belongs to a specific class of functions. Consequently, this assumption requires identifying a set of basis functions in advance, which can be a nontrivial task. There are two main classes of cost functions considered: (1) linear combinations of known basis functions called features \citep{Abbeel2004, Syed2007, Ziebart2008} and (2) convex combinations of a set of vectors \citep{Syed2008, Kamoutsi2021}. Building on the AL formalism, \cite{Syed2007} presented a game-theoretic view of AL and solved it using a multiplicative weights algorithm. Later, \cite{Syed2008} proposed a linear programming approach to solve the AL problem without employing IRL or RL as a subroutine. This marked an initial step toward leveraging the tools of mathematical optimization to address the LfD problem. Following this direction, \cite{Kamoutsi2021} introduced a convex-analytic approach to the LfD problem within the AL formalism using a generative-model oracle for the MDP's transitions. They formulated a bilinear min-max problem using Lagrangian duality and solved it using stochastic mirror descent (SMD). Moreover, \cite{Ho2016} solved the LfD problem for a general class of cost functions, solving an entropy-regularized-min-max problem, and connected their approach with generative adversarial networks. Nevertheless, this min-max problem is nonconvex-nonconcave, limiting its theoretical understanding. \par

\textbf{Contribution.} We revisit IO's tools for IRL and present the inverse problem for estimating the cost function of an MDP given an optimal policy \eqref{eq:IO-IRL} \citep{Erkin2010, Chan2023} and incorporate prior beliefs on the structure of the cost function \eqref{eq:IRL-IOproj}. Through this approach, we revisit the proof that the inverse-feasible set of this inverse problem is equivalent to the dual problem derived by \cite{Kamoutsi2021} and extend it to a general class of cost functions. Furthermore, we relax the assumption of expert optimality in \eqref{eq:IRL-IOproj}, propose a new problem tailored for suboptimal experts \eqref{eq:IO-AL}, and characterize its optimal solution. Using Lagrangian duality, we derive a regularized-convex-concave-min-max problem \eqref{eq:newproblem} for solving \eqref{eq:IO-AL}, which reduces to previous formulations \citep{Kamoutsi2021} when the regularization term is null. Additionally, we introduce Algorithm \ref{alg:our_alg} (SMD-RLfD) by showing that the stochastic mirror descent algorithm proposed in \cite{Jin2020} to solve $\ell_\infty$-$\ell_1$ games naturally adapts to our problem, provide theoretical convergence bounds, and establish a connection between the output of SMD-RLfD and the optimal solution of \eqref{eq:IO-AL}. \par

\subsection{Notation}
We denote the cardinality of a set $\St$ as \( |\St| \). The probability simplex over $|\St|$ elements is given by $\Delta^{|\St|} = \left\{ \vx \in \R^{|\St|} \mid \evx_i \geq 0, \sum_{i=1}^{|\St|} \evx_i = 1 \right\}$ and boxes are denoted by $\sB_b^{n} = \{\vx\in\R^n \mid \Vert \vx \Vert_\infty \leq b\}$. The canonical basis vectors are denoted by $\ve^{i} = \{\vx \in \R^n \mid \eve_i = 1 \text{ and } \eve_j=0 \: \forall j\neq i\}$. The Kronecker delta is denoted by $\delta_{ij}$. Component-wise multiplication between two vectors $\vx,\vy$ is denoted by $\vx\circ\vy$. The nonnegative real numbers are denoted by $\R_+$. The superscript $(\cdot)^*$ indicates optimality, for instance $\rho_\vc^*$ denotes the optimal discounted expected cost of an MDP. Finally, we use the same notation as \cite{Kamoutsi2021}, with minor modifications, to highlight the strong connections to their work.

\section{Preliminaries and problem formulation} \label{sec:preliminaries}
In this section, we establish the foundational concepts necessary for our study. We begin by defining the structure of infinite-horizon MDPs. We then introduce the IRL problem and discuss the LfD problem through the AL formalism. Finally, we provide an overview of IO and formally state our problem.

\subsection{Infinite Horizon MDPs}
A finite MDP is defined as a tuple $(\St,\Ac, P, \vnu_0, \vc, \gamma)$ where $\St$ is a finite state space, $\Ac$ a finite action space, and $P$ is a transition law $P = (P(\cdot \:|\: s,a))_{s,a}$ where $P(\cdot\:|\: s,a) \in \Delta^{|\St|}$. The initial state distribution is denoted by $\vnu_0 \in \Delta^{|\St|}$ and satisfies $\vnu_0(s) > 0$ for every $s\in\St$. The cost vector is defined as $\vc \in \C \subseteq \sB_1^{|\St||\Ac|}$  and the discount factor is given by $\gamma \in (0,1)$. \par

A \textit{stationary Markov policy} is a collection of distributions, indexed by $s\in\St$ and denoted by $(\pi(\cdot\:|\:s))_{s\in\St}$, where $\pi(\cdot\:|\:s) \in \Delta^{|\Ac|}$. We denote the space of stationary Markov policies by $\Pi_0$. In this framework, the MDP begins with an initial state $s_0\sim \vnu_0$. At each time-step $t$, where the current state is $s_t$: the agent selects an action according to $a_t\sim\pi(\cdot\:|\: s_t)$, the next state is determined by the transition law $s_{t+1}\sim P(\cdot\:|\: s_t, a_t)$, and a cost $\evc(s_t,a_t)$ is incurred. Note that in an infinite horizon model, the process continues indefinitely.\par 

The \textit{normalized value function} $\vV^\pi_{\vc} \in \R^{|\St|}$ of a policy $\pi$ given a cost $\vc$ is given by
\begin{align*}
    \vV^\pi_{\vc}(s) = (1-\gamma)\E^\pi_s \left[ \sum^\infty_{t=0} \gamma^t \evc(s_t,a_t)\right]
\end{align*}
where $\E^\pi_s[\cdot]$ denotes the expectation with respect to the trajectories generated by $\pi$ when starting from the state $s$. While we sometimes refer to it as a function, any function from a finite set to the reals can be naturally represented as a vector. The fundamental goal of RL is to find a policy $\pi$ such that the process $((s_t,a_t))_t$ minimizes the \textit{total expected cost}:
\begin{align*}
      \rho^*_\vc &= \min_{\pi\in\Pi_0} \rho_{\vc}(\pi) \tag{$\text{RL}_\vc$} \label{eq:RLgoal}\\
      &= \min_{\pi\in\Pi_0} (1-\gamma) \E^\pi_{\vnu_0} \left[ \sum\limits_{t=0}^{\infty} \gamma^t \evc(s_t, a_t) \right],
\end{align*}
where $\rho_{\vc}(\pi) = \langle \vnu_0, \vV^\pi_{\vc}\rangle$. Notice that we explicitly highlight the dependence of \eqref{eq:RLgoal} on the cost vector $\vc$. Furthermore, we denote $\vV^*_\vc$ as the value function corresponding to the optimal policy for $\text{RL}_\vc$. \par

The \textit{normalized occupancy measure} $\vmu_\pi \in \Delta^{|\St||\Ac|}$ of a policy $\pi$ is defined as 
\begin{align*}
    \vmu_\pi(s,a) = (1-\gamma)\sum\limits_{t=0}^\infty \gamma^t \sP^\pi_{\vnu_0}[s_t =s,\: a_t = a],
\end{align*}
where $\sP^\pi_{\vnu_0}[\cdot]$ represents the probability of an event when starting from $s\sim\vnu_0$ and following $\pi$. The occupancy measure of a state-action pair can be interpreted as the discounted visitation frequency of the pair when following a particular policy. Hence, we can also write $\rho^*_\vc = \min_{\pi\in\Pi_0}\langle \vmu_\pi, \vc \rangle$.\par

We define the transition matrix $\mP \in \R^{|\St|\times |\St||\Ac|}$ where $\mP_{s',(s,a)} = P(s'\:|\:s,a) $ and the polyhedron $\F = \{\vmu\in\R^{|\St||\Ac|} \mid \mTg\vmu = \vnu_0, \: \vmu\geq \vzero\}$ where $\mT \in \R^{|\St|\times |\St||\Ac|}$, $\mT_{s',(s,a)} = \delta_{s',s} - \gamma \mP_{s',(s,a)}$, and $\mTg = \frac{1}{(1-\gamma)}\mT$. An alternative expression for $\mTg$ that is useful for computing gradient estimators is $\mTg\vmu = \frac{1}{(1-\gamma)}(\mB- \gamma\mP)\vmu$ where $\mB$ is a binary matrix that satisfies $\mB_{s',(s,a)} = 1$ if $s'=s$ and $\mB_{s',(s,a)} = 0$ otherwise.\par

\begin{proposition}[\cite{Puterman1994}] \label{prop:puterman}
It holds that, $\F = \{\vmu_{\pi} \mid \pi \in \Pi_0\}$. 
For every \( \pi \in \Pi_0 \), we have that \( \vmu_{\pi} \in \F \). Moreover, for every feasible solution \( \vmu \in \F \), we can obtain a stationary Markov policy \( \pi_{\vmu} \in \Pi_0 \) by $
\pi_{\vmu}(a \mid x) = \frac{\evmu(x,a)}{\sum_{a' \in \Ac} \evmu(x,a')}$.
Then, the induced occupancy measure is exactly \( \mu \).
\end{proposition}

Proposition \ref{prop:puterman} provides a correspondence between the elements of $\F$ and occupancy measures given by stationary Markov policies. Note that the condition $\mTg\vmu = \vnu_0$ can be interpreted as the Markov property of the process under $\sP^\pi_{\vnu_0}[\cdot]$. Hence, the MDP linear programming approach consists of solving the \eqref{eq:MDP-P} problem
\begin{align*}
    \rho^*_\vc=\min\limits_{\vmu \in \Delta^{|\St||\Ac|}} &\:  \langle \vmu, \vc \rangle \\
    \text{s.t}\quad &\mTg \vmu = \vnu_0, \tag{$\text{MDP-P}_\vc$} \label{eq:MDP-P}\\
    &\vmu \geq \vzero.
\end{align*}
Note that the constraints enforce that $\vmu \in \mathcal{F}$, therefore an optimal $\vmu$ corresponds to an optimal policy. The corresponding dual problem is given by
\begin{align}
   \max_{\vu \in \R^{|\St|}}\{ \: \langle \vnu_0, \vu \rangle \mid 
     \vc - \mTg^\top \vu \geq \vzero\}, \tag{$\text{MDP-D}_\vc$} \label{eq:MDP-D}
\end{align}
where an optimal $\vu$ represents the optimal value function $\vV^*_\vc$.

\subsection{Inverse reinforcement learning}\label{subsection:IRL}
The IRL problem aims to uncover the true cost function that an expert agent is optimizing given some information about the expert's behavior, for example: sample trajectories, its real policy, or an estimate of its policy \citep{Ng2000}. Formally, given an MDP without a cost vector and with access to information about an expert's policy $\pi_E$,  the IRL problem is defined by the tuple $(\St,\Ac,P, \vnu_0, \pi_{E},\gamma)$ and the goal is to determine a cost vector $\vc$ for which the policy $\pi_E$ is optimal for \eqref{eq:RLgoal} within the MDP $(\St,\Ac,P, \vnu_0, \vc,\gamma)$. \par 

\subsection{Learning from demonstrations and the apprenticeship learning formalism}
The goal of learning from demonstrations is to learn a policy that matches or outperforms the expert's policy $\pi_E$ for an unknown true cost vector $\vc_{\text{true}}$. The apprenticeship learning formalism \citep{Abbeel2004} has been routinely used in literature for addressing the LfD problem. The AL formalism assumes that the unknown true cost function $\vc_{\text{true}}$ belongs to a class of functions $\C$ and searches for a policy that solves the following min-max problem
\begin{align*}
    \beta^* := \min_{\pi\in\Pi_0}\max_{\vc\in \C}  \langle \vmu_{\pi}, \vc \rangle - \langle \vmu_{\pi_E}, \vc \rangle = \min_{\pi\in\Pi_0}\max_{\vc\in \C} \langle \vmu_{\pi} - \vmu_{\pi_E}, \vc \rangle, \tag{$\text{LfD}_{\pi_E}$} \label{eq:lfd}
\end{align*}
An optimal solution to \eqref{eq:lfd} is called an apprentice policy $\pi_A$ and satisfies $$\langle \vmu_{\pi_A}, \vc_{\text{true}} \rangle \leq \langle \vmu_{\pi_E}, \vc_{\text{true}} \rangle + \beta^*.$$\par

In optimization-focused approaches to LfD, the $\vc_{\text{true}}$ is assumed to belong to a convex hull \begin{align*}
\mathcal{C} = \mathcal{C}_{\text{conv}} := \left\{\vc_{\vw} := \sum\limits_{i=1}^{n_c} \evw_i\vc_i \: \Big| \: \evw_i \geq 0, \sum\limits_{i=1}^{n_c} \evw_i = 1\right\}
\end{align*}
\citep{Syed2008, Kamoutsi2021} of a set of vectors $\{\vc_i\}_{i=1}^{n_c} \subseteq \R^{|\St||\Ac|}$ where $\Vert \vc_i \Vert_\infty \leq 1$ for each $i=1,...,n_c$. It is assumed that this set of vectors is known; however, in practice, an initial estimation step is required to determine this set, a task that is generally nontrivial.

\subsection{A primer on inverse optimization} \label{sec:primerIO}
Inverse optimization is a mathematical framework that fits optimization models to decision data. Given an observed optimal solution, it seeks to learn the objectives and constraints of the underlying model. For example, IRL can be thought of as an inverse optimization problem, as it searches for the cost function that an optimal agent is optimizing. \par 

Consider the general forward optimization problem \eqref{eq:FOP} for a given parameter $\theta$ in the parameter space $\Gamma$:
\begin{align*}
    \min\limits_{\vx \in \R^n}\{f(&\vx, \theta) \:|\: \vx\in X(\theta)\}, \tag{$\text{FOP}_{\theta}$} \label{eq:FOP}
\end{align*}
where $X(\theta)$ denotes the feasible set for $\vx$, which depends on $\theta$. Given an optimal solution $\hat{\vx}$, the inverse optimization problem consists of finding a $\theta^*\in\Gamma$ that makes $\hat{\vx}$ optimal for \eqref{eq:FOP} with $\theta = \theta^*$ and is optimal in some way. For this purpose, define the optimal solution set  
$X^{\text{opt}}(\theta) := \argmin\limits_{\vx}\{f(\vx,\theta) \: | \: \vx \in X(\theta)\}$
and the inverse-feasible set
$\Theta^{\text{inv}}(\hat{\vx}) := \{\theta \in\Gamma | \: \hat{\vx}\in X^{\text{opt}}(\theta)\}$. Naturally, we want to find a $\theta \in \Theta^{\text{inv}}(\hat{\vx})$, but rather than selecting an arbitrary $\theta$ from this set, we aim for one that minimizes a certain criterion. \par

Hence, the inverse optimization problem \eqref{eq:INVOPT} is defined as:
\begin{align*}
    \min\limits_{\theta \in \Gamma}\{F(\theta) \mid \theta &\in \Theta^{\text{inv}}(\hat{\vx})\}, \tag{\text{INV-OPT}} \label{eq:INVOPT}
\end{align*}
where $F$ should convey information about the quality of $\theta$ given some prior knowledge, and the search space $\Gamma$ should be appropriately chosen for each instance of the problem.

\subsection{Our problem}

Suppose that the environment is modeled as an MDP where only the state space $\St$, action space $\Ac$, and discount factor $\gamma$ are known. We assume that the learner has access to a generative-model oracle for the MDP's transition dynamics, as well as a generative-model oracle of an expert's occupancy measure $\vmu_{\pi_E}$ (not necessarily optimal), and a prior belief $\hat{\vc}$ of the cost function the expert is trying to optimize for. We aim to learn a cost function $\vc_A$ and an apprentice policy $\pi_A$, such that $\pi_A$ is optimal for $\text{RL}_{\vc_A}$, and $\vc_A$ remains close to the prior $\hat{\vc}$, while $\pi_A$ performs at least as well as $\pi_E$ under $\vc_A$ (see \eqref{eq:IO-AL} below). 

To address this problem, we use Algorithm \ref{alg:our_alg} (SMD-RLfD), which employs the expert and transition oracles alongside the proxy cost vector $\hat{\vc}$ to generate expected $\epsilon$-approximate solutions of \eqref{eq:newproblem}, which is the unconstrained formulation of \eqref{eq:IO-AL}. Hereafter, we refer to our framework as the complete process of fixing a proxy cost vector, accessing the necessary oracles, and applying SMD-RLfD to solve \eqref{eq:newproblem} and obtain the resulting learned policy and cost vector.

\section{The inverse optimization viewpoint}
In this section, we show how the IRL problem has been addressed with the tools of IO and establish the equivalence between the AL  formulation presented in \citet{Kamoutsi2021} and the inverse-feasible set of problem \eqref{eq:IO-IRL}. Afterwards, we demonstrate how prior beliefs about the structure of the cost vector can be incorporated into the IRL formulation \eqref{eq:IRL-IOproj} and present our primary contribution: the formulation of problem \eqref{eq:IO-AL} and the characterization of its optimal solution.

\subsection{IRL via IO}
We will use the ideas of Subsection \ref{sec:primerIO} applied to the forward optimization problem $\text{MDP-P}_{\vc_{\text{true}}}$, where the parameter $\theta$ corresponds to the true cost vector $\vc_{\text{true}}$ the expert is optimizing for. Note that we assume the existence of $\vc_{\text{true}}$ because the IRL problem assumes that the expert is optimal for some cost function. Therefore, let us suppose that $\vc_{\text{true}}$ lies in a convex class of cost functions $\C$ and that the expert's policy $\pi_E$ is optimal for $\text{RL}_{\vc_\text{true}}$, which means that its corresponding occupancy measure $\vmu_{\pi_E}$ is optimal for $\text{MDP-P}_{\vc_{\text{true}}}$. The following proposition follows from complementary slackness for linear problems \citep{Bertsimas1997}.
\begin{proposition}[Complementary slackness] \label{prop:complementaryslackness}
    An element $\vmu_{\pi}$ is an optimal solution to \eqref{eq:MDP-P} if and only if there exists a vector $\vu \in \R^{|\St|}$ such that $\vc - \mTg^\top \vu \geq \vzero$ and $\langle \vmu_{\pi}, \vc - \mTg^\top \vu \rangle = 0$.
\end{proposition}
Remember that $\vu$ is the dual variable for the equality constraint in \eqref{eq:MDP-P} and represents the value function. Therefore, the inverse-feasible set for $\vmu_{\pi_E}$ consists of the cost functions in $\C$ for which such a $\vu$ exists
\begin{align*}
    \Theta^{\text{inv}}(\vmu_{\pi_E}) := \{\vc \in \C \: | \: \exists \vu \in \R^{|\St|} \: : \: \vc - \mTg^\top \vu \geq \vzero, \: \langle \vmu_{\pi_E}, \vc - \mTg^\top \vu \rangle = 0\}.
\end{align*}
Substituting for the inverse-feasible set in \eqref{eq:INVOPT} and choosing an appropriate function $F$ for comparing cost vectors, we arrive to the inverse reinforcement learning problem through inverse optimization \citep{Erkin2010, Chan2023}
\begin{align*}
    \min\limits_{\vc\in \C, \vu \in \R^{|\St|} } \quad F(\vc)&
     \\
    \text{s.t} \quad \quad \quad  \vc - \mTg^\top \vu &\geq \vzero, \tag{IRL-IO} \label{eq:IO-IRL}\\
    \langle\vmu_{\pi_E}, \vc - \mTg^\top \vu  \rangle &= 0.
\end{align*}

\subsection{Connections to LfD and the AL formalism} \label{sec:Lfd-AL}
\cite{Kamoutsi2021} considered the LfD problem under the assumption that the true cost function $\vc_{\text{true}}$ belongs to the convex hull $\C_{\text{conv}}$ of a given set of vectors. By applying an epigraphic transformation to \eqref{eq:lfd}, where the validity of this transformation depends on the convex hull assumption on $\vc_{\text{true}}$, and deriving its dual, they arrived at the optimization problem \eqref{eq:dualAL}:
\begin{align*}
    \max_{\vc, \vu} \{ \langle \vmu_{\pi_E}, \mTg^\top \vu - \vc \rangle \: | \: \vc\in\C_{\text{conv}}, \vc - \mTg^\top \vu \geq \vzero\}. \tag{$\text{D}_{\pi_E}$} \label{eq:dualAL}
\end{align*}
They focus on this problem and optimize its unconstrained version derived through Lagrangian duality, where the dual variable corresponding to the constraint $\vc - \mTg^\top \vu \geq \vzero$ represents the apprentice state-action visitation probability.
\begin{theorem}[cf. Proposition 2 in \cite{Kamoutsi2021}]\label{theorem:IO-IRL}
    Suppose that $\vmu_{\pi_E}$ is an optimal solution for \eqref{eq:MDP-P} where $\vc \in \C$. Then the following equality holds:
    $$\Theta^{\text{inv}}(\vmu_{\pi_E}) = \Pi_1\left( \argmax\limits_{(\vc,\vu)} \{ \langle \vmu_{\pi_E}, \mTg^\top \vu - \vc \rangle \: | \: \vc \in \C, \: \vc - \mTg^\top\vu \geq 0 \}\right)$$
    where $\Pi_1$ denotes the projection in the first component.
\end{theorem}

This implies that the dual problem \eqref{eq:dualAL} serves as an alternative representation of the inverse-feasible set $\Theta^{\text{inv}}(\vmu_{\pi_E})$. In contrast, in problem \eqref{eq:IO-IRL} we choose an element within the inverse-feasible set that minimizes $F$. In this sense, under the assumption of expert's optimality and that $\vc_{\text{true}} \in \C_{\text{conv}}$, the AL formalism finds an arbitrary element of the inverse-feasible set, whereas \eqref{eq:IO-IRL} has a criterion for searching within this space and considers a general convex class of cost functions $\C$.

\subsection{Incorporating prior beliefs}

Suppose we are given a proxy cost vector $\hat{\vc}$ that reflects our prior beliefs about the structure of the true cost vector, which are not necessarily accurate. Leveraging this information, we aim to guide the search within the inverse feasible set in problem \eqref{eq:IO-IRL}. To this end, we project $\hat{\vc}$ onto $\Theta^{\text{inv}}(\vmu_{\pi_E})$ by solving the following optimization problem:
\begin{align}
    \min\limits_{\vc\in \C, \vu \in \R^{|\St|} } \quad \Vert \vc - \hat{\vc}\Vert^2_2& \nonumber
     \\
    \text{s.t} \quad \quad \quad  \vc - \mTg^\top \vu &\geq \vzero,\tag{$\text{IRL-IO}_{\hat{\vc}}$} \label{eq:IRL-IOproj}\\
    \langle\vmu_{\pi_E}, \vc - \mTg^\top \vu  \rangle &= 0 \nonumber .
\end{align}
Figure \ref{fig:illustration_projection} illustrates the setting for problem \eqref{eq:IRL-IOproj}. The big polyhedral region represents the set of all occupancy measures $\F$. In particular, if we assume that the expert is optimal and the induced policy is deterministic, then $\vmu_{\pi_E}$ is a vertex of $\F$ and the true cost vector $\vc_{\text{true}}$, in green, is within the inverse-feasible set $\Theta^{\text{inv}}(\vmu_{\pi_E})$. The proxy cost vector $\hat{\vc}$, in red, is not necessarily inside the inverse-feasible set and will be projected onto the inverse-feasible set by solving Problem \eqref{eq:IRL-IOproj}. This yields the cost vector closest to our prior belief that satisfies the expert's optimality conditions.
\begin{figure}[ht]
    \centering
    \begin{tikzpicture}
        \draw[-{Stealth},red] (0,0) -- (1,0.2); 
        \draw[red] (1,0.0) node [label={}] {\small{$\hat{\vc}$}};
        
        \draw[-{Stealth},OliveGreen] (0,0) -- (0.45,0.5); 
        \draw[OliveGreen] (0.3,0.6) node [label={}] {\small{$c$}};
        
        \draw (0,0)--(0,1.7);
        \draw (0,0)--(2.3,1.2);
        \draw (0,0.92)--(0.5,0.92);
        \draw (0.70,0.35)--(0.5,0.92);
        
        \draw[smooth,tension=1, -{Stealth}] plot coordinates{ (0.65,1.45) (0.45,1.3) (0.35,0.8) };
        
        \draw (1.5,1.5) node [label={}] {\small{$\Theta^{\text{\tiny{inv}}}(\vmu_{\pi_{\text{\tiny{E}}}})$}};
        
        \draw[Fuchsia] (-3.0,0)--(0,0) ;
        
        \draw[Fuchsia] (0,0)--(0.6,-1) ;
        \draw[Fuchsia] (-0.3,-2.2)--(0.6,-1) ;
        
        \draw (0,0) node [label={}] {\textbullet};
        \draw (-0.3,-0.3) node [label={}] {\small{$\vmu_{\pi_{\text{\tiny{E}}}}$}};
        \draw[Fuchsia] (-1,-1) node [label={}] {\small{$\mathcal{F}$}};
        
    \end{tikzpicture}
    \caption{Illustration of \eqref{eq:IRL-IOproj}.}
    \label{fig:illustration_projection}
\end{figure}

Since we are still selecting an element from the inverse-feasible set, we can derive results similar to those presented by \cite{Kamoutsi2021} for solutions to Problem \eqref{eq:IRL-IOproj} under the assumption of expert optimality. The following corollary follows directly from the proof of Theorem \ref{theorem:IO-IRL}. %(see Appendix \ref{sec:appendix2}):
\begin{corollary}[Optimal expert] \label{cor:optimalExpert}
    Assume that $\pi_E$ is optimal for \eqref{eq:RLgoal} with $\vc\in\C$. A pair $(\vc_A,\vu_A)$ is optimal for Problem \eqref{eq:IRL-IOproj} if and only if $\pi_E$ is optimal for \eqref{eq:RLgoal} with $\vc=\vc_A$ and $\vu_A = \vV_{\vc_A}^*$. In particular, the cost vector $\vc_A$ is the projection of $\hat{\vc}$ onto the inverse feasible set, and $\pi_E$ is optimal for $\vc_A$.
\end{corollary}

It is important to note that when the expert is suboptimal, problem \eqref{eq:IRL-IOproj} is infeasible, as the complementary slackness equality cannot be satisfied. To account for the possibility of suboptimal expert behavior, we can relax the complementary slackness condition in \eqref{eq:IRL-IOproj}. Weighing the beliefs of the expert's optimality and the quality of the cost function estimate with parameter $\alpha\in\R_{+}$, we arrive to problem \eqref{eq:IO-AL}:
\begin{align*}
\min\limits_{\vc\in \C, \vu \in \R^{|\St|} } \quad \alpha\Vert \vc - \hat{\vc} \Vert^2_2& + \langle \vmu_{\pi_E}, \vc - \mTg^\top \vu \rangle
 \\
\text{s.t} \quad \quad \quad \hspace{7px} \vc - \mTg^\top& \vu \geq \vzero. \tag{$\text{IO-AL}_\alpha$} \label{eq:IO-AL}
\end{align*}

\begin{proposition}[Suboptimal expert] \label{prop:suboptimalexpert}
    A pair $(\vc_A, \vu_A)$ is optimal for \eqref{eq:IO-AL} if and only if the apprentice policy $\pi_A$ is optimal for \eqref{eq:RLgoal} with $\vc=\vc_A$ and $\vu_A = \vV^*_{\vc_A}$. Furthermore, the optimal value corresponds to $\alpha\Vert \vc_A - \hat{\vc} \Vert^2_2 + \rho_{\vc_A}(\pi_E) - \rho_{\vc_A}(\pi_A)$.
\end{proposition}
\begin{proof}
    Observe that \eqref{eq:IO-AL} is a convex optimization problem. Therefore, a pair $(\vc_A, \vu_A)$ is optimal for \eqref{eq:IO-AL} if and only if the KKT conditions are satisfied. In particular, it satisfies complementary slackness $\langle\vmu_{\pi_A}, \vc_A - \mTg^\top \vu_{A}\rangle = 0$ and the stationarity condition $\mTg \vmu_{\pi_A}=\mTg \vmu_{\pi_E}=\vnu_0$, which implies that $\vmu_{\pi_A}\in\F$. Then, by Proposition \ref{prop:complementaryslackness} $\pi_A$ is optimal for \eqref{eq:RLgoal} with $\vc=\vc_A$ and $\vu_A = \vV_{\vc_A}^*$. Moreover, if $(\vc_A, \vu_A)$ are optimal we have that
    \begin{align*}
        \alpha\Vert \vc_A - \hat{\vc} \Vert^2_2 + \langle \vmu_{\pi_E}, \vc_A - \mTg^\top \vu_A \rangle &= \alpha\Vert \vc_A - \hat{\vc} \Vert^2_2 + \langle \vmu_{\pi_E}, \vc_A \rangle - \langle \vmu_{\pi_E}, \mTg^\top \vu_A\rangle\\
        &= \alpha\Vert \vc_A - \hat{\vc} \Vert^2_2 + \rho_{\vc_A}(\pi_E) - \langle \vnu_0, \vu_A\rangle\\
        &= \alpha\Vert \vc_A - \hat{\vc} \Vert^2_2 + \rho_{\vc_A}(\pi_E) - \rho_{\vc_A}^*.
    \end{align*}
\end{proof}
Proposition \ref{prop:suboptimalexpert} states that the apprentice policy $\pi_A$, i.e., the dual variable for the constraint $\vc - \mTg^\top \vu \geq \vzero$, is optimal for the \eqref{eq:RLgoal} with cost vector $\vc_A$. Solutions to \eqref{eq:IO-AL} can be viewed as a way to balance the distance between the cost vector $\vc_A$ and the estimate $\hat{\vc}$, while ensuring that the total expected cost of $\pi_E$ and $\pi_A$ are similar under $\vc_A$. \par
\begin{figure}[ht]
    \centering
    \begin{tikzpicture}
        \draw[-{Stealth},red] (0,0) -- (1,0.2); 
        \draw[red] (1,0.0) node [label={}] {\small{$\hat{\vc}$}};
        \draw[-{Stealth},red] (0.6,-1) -- (1.6,-0.8); 
        \draw[red] (1.55,-1) node [label={}] {\small{$\hat{\vc}$}};

        \draw (0,0)--(0,1.7);
        \draw (0,0)--(1.7,1.2);
        \draw (0,0.92)--(0.5,0.92);
        \draw (0.70,0.48)--(0.5,0.92);

        \draw (0.6,-1)--(2.25,-0.2);
        \draw (0.6,-1)--(2.25,-1.8);
        
        \draw (1.5,-0.57)--(1.7,-1.05);
        \draw (1.7,-1.05)--(1.6,-1.5);

        \draw[Fuchsia] (-3.0,0)--(0,0) ;
        
        \draw[Fuchsia] (0,0)--(0.6,-1) ;
        \draw[Fuchsia] (-0.3,-2.2)--(0.6,-1) ;
        
        \draw (0,0) node [label={}] {\textbullet};
        \draw (-0.2,-0.2) node [label={}] {\small{$\vmu_{\pi_{A}}$}};
        \draw (0.6,-1) node [label={}] {\textbullet};
        \draw (0.17,-1) node [label={}] {\small{$\vmu_{\pi_{A}'}$}};
        
        \draw (-1.2,-0.7) node [label={}] {\small{$\vmu_{\pi_{\text{\tiny{E}}}}$}};
        \draw (-1,-0.5) node [label={}] {\textbullet};

        \draw[Fuchsia] (-2,-1) node [label={}] {\small{$\mathcal{F}$}};
        
    \end{tikzpicture}
    \caption{Illustration of \eqref{eq:IO-AL}.}
    \label{fig:illustration_IOAL}
\end{figure}

Figure \ref{fig:illustration_IOAL} illustrates \eqref{eq:IO-AL} using the same notation as Figure \ref{fig:illustration_projection}. In this scenario, the suboptimal expert lies within the occupancy measure set $\mathcal{F}$ rather than at a vertex. The apprentice policy $\pi_A$ provides a better explanation of the expert's behavior, as it corresponds to the closest vertex, while $\pi_A'$ better aligns with the proxy cost vector $\hat{\vc}$. The parameter $\alpha$ governs this trade-off: when $\alpha$ increases, the optimization problem selects $\vmu_{\pi_A'}$, whereas for values closer to zero, it selects $\vmu_{\pi_A}$. In simple terms, $\alpha$ dictates the level of importance we place in our prior information versus the demonstrations provided by the suboptimal expert. \par

This is our alternative to the usual AL formalism: instead of identifying the set of vectors $\{\vc_i\}_{i=1}^{n_c}$ that define $\C_{\text{conv}}$ and choosing an arbitrary cost vector within the inverse-feasible set, we search over a general convex class of cost vectors $\C$ and define an estimate $\hat{\vc}$ to guide the search. In practice, obtaining information about optimal experts is challenging \citep{Brown2019, Chen2021, Wang2021}, and given an expert's policy or demonstrations, it is difficult to determine whether it is optimal. Nonetheless, we still aim to leverage suboptimal experts' information and comprehend its actions by solving \ref{eq:IO-AL}.

\subsection{Min-max formulation}
We aim to reformulate \eqref{eq:IO-AL} as a convex-concave-min-max problem and solve this unconstrained optimization problem using stochastic mirror descent. To this end, we compute its Lagrangian:
\begin{align*}
    \mathcal{L}(\vc, \vu, \vmu)  &= \alpha\Vert \vc - \hat{\vc} \Vert^2_2 + \langle \vmu_{\pi_E} -\vmu, \vc - \mTg^\top \vu \rangle
\end{align*}
where $\vmu \in \R^{|\St||\Ac|}$ and $\vmu \geq \vzero$. Observe that $\mathcal{L}(\vc,\vu,\vmu)$ is convex on $(\vc,\vu)$ and concave on $\vmu$. Thus, \eqref{eq:IO-AL} is equivalent to the min-max problem
\begin{align*}
    \min\limits_{\vc\in\C, \vu\in\R^{|\St|}}\max_{\vmu \geq \vzero} \mathcal{L}(\vc, \vu, \vmu).
\end{align*}
\par 
In our setting, we assume that $\C= \sB^{|\St||\Ac|}_1$, which is not restrictive because we can scale any cost vector to lie within this set. Therefore, we know that $\Vert \vV^\pi_{\vc}\Vert_\infty \leq 1$  for any policy $\pi \in \Pi_0$ and $\vc\in\C$ (see Lemma \ref{lemma:boundValueFunction} in the Appendix). 
Hence, we can search for $(\vc,\vu)$ within the box $\sB^{|\St||\Ac|}_1 \times \sB^{|\St|}_1$. Moreover, as all feasible solutions for \eqref{eq:MDP-P} belong to the simplex $\Delta^{|\St||\Ac|}$, we can restrict the search for $\vmu$ to the same simplex.
\begin{align}
    \min\limits_{(\vc,\vu) \in \sB^{|\St||\Ac|}_{1}\times \sB^{|\St|}_{1}}\max_{\vmu \in \Delta^{|\St||\Ac|}} \alpha\Vert \vc - \hat{\vc} \Vert_2^2 + \langle \vmu_{\pi_E} -\vmu, \vc - \mTg^\top \vu \rangle.  \tag{$\text{RLfD}_\alpha$}\label{eq:newproblem}
\end{align}
Observe that this formulation closely resembles previous min-max formulations of the LfD problem \citep{Kamoutsi2021}. It can be interpreted as a regularized version of this problem, where the search for $\vc$ is conducted within a general class of cost functions rather than being restricted to a previously specified convex hull.

\section{Algorithm}

Revisiting the assumptions for our problem, we assume that we have access to a generative-model oracle of the expert's occupancy measure $\vmu_{\pi_E}$, as well as a generative-model oracle for the MDP's transition law. In this section, we will focus on solving \eqref{eq:newproblem} via stochastic mirror descent. Before attempting to solve this problem, we must first define what constitutes a good solution. We define an $\epsilon$-approximate solution as a pair  $(\vc,\vu), \vmu$ such that their duality gap is bounded by $\epsilon >0$.
\begin{definition}[$\epsilon$-approximate solution]
    Given $\epsilon>0$, an $\epsilon$-approximate solution of   \eqref{eq:newproblem} is a pair of feasible solutions $((\vc^\epsilon, \vu^\epsilon), \vmu^\epsilon) \in \left(\sB^{|\St||\Ac|}_{1}\times \sB^{|\St|}_{1}\right) \times \Delta^{|\St||\Ac|}$ that satisfy 
$$\text{Gap}((\vc^\epsilon, \vu^\epsilon), \vmu^\epsilon) := \max_{\vmu'\in\Delta^{|\St||\Ac|}}\mathcal{L}((\vc^\epsilon,\vu^\epsilon),\vmu') - \min_{(\vc',\vu')\in \sB^{|\St||\Ac|}_{1}\times \sB^{|\St|}_{1}}\mathcal{L}((\vc',\vu'),\vmu^\epsilon) \leq \epsilon.$$
\end{definition}

To minimize the duality gap, we require descent and ascent directions. The gradients of $\mathcal{L}((\vc, \vu), \vmu)$ at a given iterate $((\vc_t, \vu_t), \vmu_t) \in ( \sB^{|\St||\Ac|}_{1}\times \sB^{|\St|}_{1} ) \times \Delta^{|\St||\Ac|}$ are given by
\begin{align*}
    g^{(\vc,\vu)}((\vc_t,\vu_t),\vmu_t) &= \begin{pmatrix}
2\alpha(\vc_t-\hat{\vc}) + \vmu_{\pi_E} - \vmu_t \\
\mTg\vmu_t - \mTg\vmu_{\pi_E}
\end{pmatrix},\\
    g^{\vmu}((\vc_t,\vu_t),\vmu_t) &= -(- \vc_t + \mTg^\top \vu_t) = \vc_t - \mTg^\top \vu_t,
\end{align*}
where $g^{(\vc,\vu)}((\vc_t,\vu_t),\vmu_t) = \nabla_{(\vc,\vu)}\mathcal{L}((\vc_t, \vu_t), \vmu_t)$ and $g^{\vmu}((\vc_t,\vu_t),\vmu_t) = -\nabla_{\vmu}\mathcal{L}((\vc_t, \vu_t), \vmu_t)$. Since explicit access to $\mTg$ and $\vmu_{\pi_E}$ is unavailable, it is necessary to develop gradient estimators that are compatible with oracle-based queries.

\begin{definition}[Bounded estimator] \label{def:bounded_estimator}
Given the following properties on the mean, scale, and variance of an estimator:
\begin{enumerate}
    \item unbiasedness: $\E[\Tilde{g}] = g$.
    \item bounded maximum entry: $\Vert \Tilde{g} \Vert_\infty \leq z$ with probability 1.
    \item bounded second-moment: $\E[\Vert \Tilde{g} \Vert^2] \leq v$
\end{enumerate}
we call $\Tilde{g}$ a $(v, \Vert \cdot \Vert)$-bounded estimator if it satisfies 1. and 3. and a $(z, v, \Vert \cdot \Vert_{\Delta^m})$-bounded estimator if it satisfies all conditions with local norm $\Vert \cdot \Vert_\vw$ for all $\vw\in\Delta^m$.
\end{definition}

With this in mind, define the gradient estimator for the $(\vc,\vu)$ side through the following procedure
\begin{align}\label{eq:gradCU}\nonumber
    \text{sample}\quad  &(s, a) \sim \frac{1}{|\St||\Ac|},\: (s_t, a_t) \sim \vmu_t, \: s^\prime_t \sim P(\cdot\:|\:s_t,a_t),\: \\
&(s_E, a_E) \sim \vmu_{\pi_E}, \: s^\prime_E \sim P(\cdot\:|\:s_E,a_E), \nonumber \\
    \text{set} \quad & \Tilde{g}^{(\vc,\vu)}((\vc_t,\vu_t),\vmu_t) \\
&= \begin{pmatrix}
                |\St||\Ac|\cdot 2\alpha\left(\evc_t(s,a)\ve^{(s,a)}-\hat{\evc}(s,a)\ve^{(s,a)}\right) + \ve^{(s_E,a_E)} - \ve^{(s_t,a_t)}\vspace{3px} \\\nonumber
                \frac{1}{(1-\gamma)}\left( \ve^{s_t} - \gamma \ve^{s_t^\prime} - (\ve^{s_E} - \gamma\ve^{s_E^\prime})\right)
                \end{pmatrix}.  
\end{align}
In Lemma \ref{lemma:cu_grad}, we show that this estimator is unbiased, and we provide a bound for its second moment.
\begin{lemma} \label{lemma:cu_grad}
    Gradient estimator $\Tilde{g}^{(\vc,\vu)}((\vc_t,\vu_t),\vmu_t)$ is a $(v^{(\vc,\vu)}, \Vert \cdot \Vert_2)$-bounded estimator, with 
    $$v^{(\vc,\vu)} =  64\alpha^2\cdot |\St||\Ac| + \frac{4(1+\gamma^2)}{(1-\gamma)^2} + 8.$$   
\end{lemma}

For the $\vmu$ side, define the gradient estimator by
\begin{align}\label{eq:mu_grad}\nonumber
    \text{sample} \quad &(s, a) \sim \frac{1}{|\St||\Ac|},\quad s^\prime \sim P(\cdot\:|\:s,a), \nonumber\\
    \text{set} \quad &\Tilde{g}^{\vmu}((\vc_t,\vu_t),\vmu_t)\\\nonumber
& = |\St||\Ac|\left(\evc_t(s,a)\ve^{(s,a)} - \frac{1}{(1-\gamma)}(\evu_t(s)\ve^{(s,a)}-\gamma\evu_t(s')\ve^{(s,a)})\right).
\end{align}
As before, we will demonstrate unbiasedness and bound its second moment; however, this time we will also calculate a bound on its maximum entry.
\begin{lemma} \label{lemma:mu_grad}
    Gradient estimator $\Tilde{g}^{\vmu}((\vc_t,\vu_t),\vmu_t)$ is a $(z^{\vmu}, v^{\vmu}, \Vert \cdot \Vert_{\Delta})$-bounded estimator, with 
    \begin{align*}
        z^{\vmu} =   \frac{2|\St||\Ac|}{(1-\gamma)} \: \text{and} \: v^{\vmu} = |\St||\Ac|\left(2 + \frac{4(1+\gamma^2)}{(1-\gamma)^2}\right).
    \end{align*}
\end{lemma}

Using these gradient estimators and the bounds established above, we adapt the SMD algorithm originally designed for solving MDPs in \cite{Jin2020}.  Algorithm \ref{alg:our_alg} presents the SMD method for \eqref{eq:newproblem}. This algorithm iteratively computes bounded gradient estimators (Lines 3 and 5) by sampling from the occupancy measures and querying the oracle (Lines 2 and 4). The updates are then obtained using mirror descent steps followed by a projection (Lines 6 and 7). After \( T \) iterations, the algorithm returns the average of the iterates as an $\epsilon$-approximate solution to \eqref{eq:newproblem}.

\begin{algorithm}[ht]
    \caption{Stochastic Mirror Descent for \eqref{eq:newproblem} (SMD-RLfD)}
    \label{alg:our_alg}
    \DontPrintSemicolon
    \LinesNumbered

    \textbf{Parameters:} Step-size $\eta^{(\vc, \vu)}$, $\eta^{\vmu}$, number of iterations $T$, accuracy level $\epsilon$.\\
    
    \KwIn{State space $\St$, action space $\Ac$, transition oracle $P$, occupancy measure oracle $\vmu_{\pi_E}$, initial state distribution $\vnu_0$, discount factor $\gamma$, initial $((\vc_0,\vu_0), \vmu_0) \in \sB^{|\St||\Ac|}_{1} \times \Delta^{|\St||\Ac|}$.}
    \KwOut{An expected $\epsilon$-approximate solution $((\vc^\epsilon, \vu^\epsilon), \vmu^\epsilon)$ for \eqref{eq:newproblem}.}

    \For{$t \leftarrow 0$ \KwTo $T-1$}{
        \tcc{$(\vc,\vu)$ gradient estimation}
        Sample $(s_t, a_t) \sim \vmu_t$, $s^\prime_t \sim P(\cdot \mid s_t, a_t)$, $(s_E, a_E) \sim \vmu_{\pi_E}$, $s^\prime_E \sim P(\cdot \mid s_E, a_E)$\;
        
        Compute:
        \[
        \Tilde{g}^{(\vc,\vu)}((\vc_t,\vu_t),\vmu_t) = 
        \begin{pmatrix}
            2\alpha(\vc_t-\hat{\vc}) + \vmu_{\pi_E} - \vmu_t \\[1mm]
            \frac{1}{(1-\gamma)}\left( \ve^{s_t} - \gamma \ve^{s_t^\prime} - (\ve^{s_E} - \gamma\ve^{s_E^\prime})\right)
        \end{pmatrix}
        \]

        \tcc{ $\vmu$ gradient estimation}
        Sample $(s, a) \sim \frac{1}{|\St||\Ac|}$, $s^\prime \sim P(\cdot \mid s,a)$\;
        
        Compute:
        \[
        \Tilde{g}^{\vmu}((\vc_t,\vu_t),\vmu_t) = |\St||\Ac|\left(\evc_t(s,a)\ve^{(s,a)} - \frac{1}{(1-\gamma)}(\evu_t(s)\ve^{(s,a)}-\gamma\evu_t(s')\ve^{(s,a)})\right)
        \]

        \tcc{ Mirror descent steps}
        $(\vc_{t}, \vu_{t}) \gets \Pi_{\sB^{|\St||\Ac|}_{1} \times \sB^{|\St|}_{1}}
          \left((\vc_{t-1},\vu_{t-1}) - \eta^{(\vc,\vu)}\Tilde{g}^{(\vc,\vu)}((\vc_{t-1},\vu_{t-1}), \vmu_{t-1})\right)$\;
        
        $\vmu_{t} \gets \Pi_{\Delta^{|\St||\Ac|}}
          \left( \vmu_{t-1} \circ \exp(-\eta^{\vmu} \Tilde{g}^{\vmu}((\vc_{t-1},\vu_{t-1}), \vmu_{t-1})) \right)$\;
    }
    \KwRet{$\ensuremath{((\vc^\epsilon, \vu^\epsilon), \vmu^\epsilon) \gets \frac{1}{T} \sum_{t=1}^{T} ((\vc_t, \vu_t), \vmu_t)}$}
\end{algorithm}

\begin{theorem}\label{theorem:algconvergence}
    Given $\epsilon \in (0,1)$, Algorithm \ref{alg:our_alg} with step-size
    \begin{align*}
        \eta^{(\vc,\vu)} = \frac{\epsilon}{4v^{(\vc,\vu)}}, \quad \eta^{\vmu} = \frac{\epsilon}{4v^{\vmu}},
    \end{align*} 
    and gradient estimators defined in equations \eqref{eq:gradCU} and \eqref{eq:mu_grad} finds an expected $\epsilon$-approximate solution $$\E\left[\text{Gap}((\vc^\epsilon, \vu^\epsilon), \vmu^\epsilon)\right] \leq \epsilon,$$
    within any iteration number 
    \begin{align*}
        T\geq \max\Bigg\{\mathcal{O}\left(\frac{\alpha^2 |\St|^3|\Ac|^2}{\epsilon^2}\right),\: \mathcal{O}\left(\frac{|\St||\Ac|\log(|\St||\Ac|)}{\epsilon^2}\right)\Bigg\}.
    \end{align*}
\end{theorem}

The number of iterations scales quadratically with the number of actions and cubically with the number of states. Theoretically, the number of iterations required depends on the parameter $\alpha\in \R_+$. When $\alpha = 0$, the number of iterations decreases significantly as the $|\St|^3$ and $|\Ac|^2$ terms vanish from the initial expression. This suggests that introducing the regularization $\alpha \Vert \vc - \hat{\vc} \Vert^2_2$ increases the complexity of the problem. Nevertheless, we will see in the next section that the regularization term helps to guide the search to uncover the true cost function.

\begin{proposition} \label{prop:epsilonapproxOptimality}
    Let $((\vc^\epsilon,\vu^\epsilon),\vmu^\epsilon)$ be an expected $\epsilon$-approximate solution for \eqref{eq:newproblem}, where $\vmu^\epsilon$ induces a policy $\pi_{\vmu^\epsilon}\in \Pi_0$ defined by $\pi_{\vmu^\epsilon}(a|s) = \frac{\vmu^\epsilon(a,s)}{\sum_{a'}\vmu^\epsilon(s,a)}$. It then holds that
    $$\E\left[\alpha\Vert \vc^\epsilon - \hat{\vc}\Vert^2_2 + \rho_{\vc^\epsilon}(\pi_E) -\rho_{\vc^\epsilon}(\pi_A) \right] \leq \epsilon +\alpha\Vert \vc_A - \hat{\vc}\Vert^2_2 + \rho_{\vc_A}(\pi_E) -\rho_{\vc_A}^*,$$
    where $((\vc_A,\vu_A), \vmu_A)$ denotes the optimal solution for \eqref{eq:newproblem} and $\pi_A$ is the policy induced by $\vmu_{A}$.
\end{proposition}
\begin{proof}
As $((\vc^\epsilon,\vu^\epsilon),\vmu^\epsilon)$ is an $\epsilon$-approximate solution we know that 
\begin{align*}
    \E\left[\mathcal{L}((\vc^\epsilon,\vu^\epsilon),\vmu_{A}) - \mathcal{L}((\vc_A,\vV^{\pi_{\vmu^\epsilon}}_{\vc_A}),\vmu^\epsilon) \right] \leq \epsilon.
\end{align*}
Moreover, 
\begin{align*}
    &\mathcal{L}((\vc^\epsilon,\vu^\epsilon),\vmu_{A}) - \mathcal{L}((\vc_A,\vV^{\pi_{\vmu^\epsilon}}_{\vc_A}),\vmu^\epsilon)\\
    = &\alpha\Vert \vc^\epsilon - \hat{\vc} \Vert_2^2 - \alpha\Vert \vc_A - \hat{\vc} \Vert_2^2+ \langle \vmu_{\pi_E} -\vmu_A, \vc^\epsilon - \mTg^\top \vu^\epsilon \rangle - \langle \vmu_{\pi_E} -\vmu^\epsilon, \vc_A - \mTg^\top \vV^{\pi_{\vmu^\epsilon}}_{\vc_A} \rangle\\
    \stackrel{(i)}{=}  &\alpha\Vert \vc^\epsilon - \hat{\vc} \Vert_2^2 - \alpha\Vert \vc_A - \hat{\vc} \Vert_2^2+ \rho_{\vc^\epsilon}(\pi_E) - \rho_{\vc^\epsilon}(\pi_A) - \rho_{\vc_A}(\pi_E) + \rho_{\vc_A}(\pi_{\vmu^\epsilon})\\
    \stackrel{(ii)}{\geq}  &\alpha\Vert \vc^\epsilon - \hat{\vc} \Vert_2^2 - \alpha\Vert \vc_A - \hat{\vc} \Vert_2^2+ \rho_{\vc^\epsilon}(\pi_E) - \rho_{\vc^\epsilon}(\pi_A) - \rho_{\vc_A}(\pi_E) + \rho_{\vc_A}(\pi_A),
\end{align*}
where $(i)$ follows from $\langle \vmu_{1} - \vmu_2, -\mTg^\top\vu\rangle = \langle \vnu_0, \vu\rangle - \langle \vnu_0, \vu\rangle = 0$ and Lemma 2 (Appendix B.1) in \cite{Kamoutsi2021} and 
$(ii)$ from $\rho_{\vc_A}^* = \rho_{\vc_A}(\pi_A) \leq \rho_{\vc_A}(\pi)$ for any policy $\pi\in\Pi_0$. Rearranging these terms, we arrive to
\begin{align*}
    \E\left[\alpha\Vert \vc^\epsilon - \hat{\vc}\Vert^2_2 + \rho_{\vc^\epsilon}(\pi_E) -\rho_{\vc^\epsilon}(\pi_A) \right] \leq \epsilon +\alpha\Vert \vc_A - \hat{\vc}\Vert^2_2 + \rho_{\vc_A}(\pi_E) -\rho_{\vc_A}^*.
\end{align*}
\end{proof}
Proposition \ref{prop:epsilonapproxOptimality} establishes a connection between expected $\epsilon$-approximate solutions for \eqref{eq:newproblem} and its optimal solution. Specifically, it shows that an expected $\epsilon$-approximate solution achieves an objective value that is at most $\epsilon$ worse than that of the optimal solution. Note that we do not assume that $\vmu^\epsilon$ is an occupancy measure, which is not always the case for an arbitrary $\vmu\in\Delta^{|\St||\Ac|}$.

\section{Numerical experiments}

We analyze the performance of the proposed framework using two case studies: a low-dimensional inventory management problem and a higher-dimensional Gridworld problem. The inventory example serves as a tractable system to build intuition and offer insights into the practical selection of proxy cost vector $\hat{\vc}$. Moreover, we conduct ablation analyses on prior misspecification and expert suboptimality, and execute a direct comparison with the convex hull methodology of \cite{Kamoutsi2021}. The Gridworld example is employed to demonstrate the method's convergence properties and quantify the specific impact of the regularization term on the learned cost vector in a higher-dimensional space.

To facilitate reproducibility, the complete source code and experimental setups are publicly available at \url{https://github.com/EstebanLeiva/apprenticeshiplearning.git}. Because SMD-RLfD yields expected $\epsilon$-approximate solutions for \eqref{eq:newproblem}, we account for stochasticity by executing the algorithm for $N$ independent runs. Each run uses randomly generated initial vectors $((\vc_0, \vu_0), \vmu_0)$ and proceeds for $T$ iterations; we report the average of the resulting outputs. 

\subsection{Single-product inventory control}
The single-product stochastic inventory control problem is a classic model for finite state MDPs \citep{Puterman1994}. It is defined by a state space $\mathcal{S} = \{0, 1, \dots, M\}$ representing the inventory level on hand, and an action space $\mathcal{A}(s) = \{0, 1, \dots, M-s\}$ representing the quantity ordered, where $M=15$ is the maximum capacity. The system transition from the inventory level $L = s+a$ to the next state $s' = \max(0, L - D)$ is governed by a stochastic demand $D$, which is modeled using a known Poisson distribution with mean $\lambda=10$. The goal is to minimize the expected operational cost $c(s, a)$, which is linear with respect to order cost $ c_o=3$, holding cost $ c_h =0.5$, and sell price $c_s=15$. This cost is defined by the equation $c(s, a) = c_o \cdot a + c_h \cdot (s + a) - c_s \cdot \mathbb{E}[\min(s+a, D)]$, where $\mathbb{E}[\min(s+a, D)]$ is the expected number of units sold. Note that the dimension of the cost vector depends on $M$, but it is completely determined by the three parameters above. At each timestep, the expert observes the inventory level $s$ and selects an order quantity $a$ (assuming instantaneous replenishment), after which demand $d$ is fulfilled with no backlog, and the system updates the inventory level for the next iteration. Finally, the discount factor is fixed to 0.9, and initial states are chosen uniformly.

Solving \eqref{eq:MDP-P} to optimality yields a policy that consistently restocks to a level of 14 units (Figure \ref{fig:inventory_opt_pol}). The policy avoids utilizing the full capacity, which can be understood as a consequence of the relationship between holding costs and expected demand.
\begin{figure}[ht]
    \centering
    \includegraphics[width=0.8\linewidth]{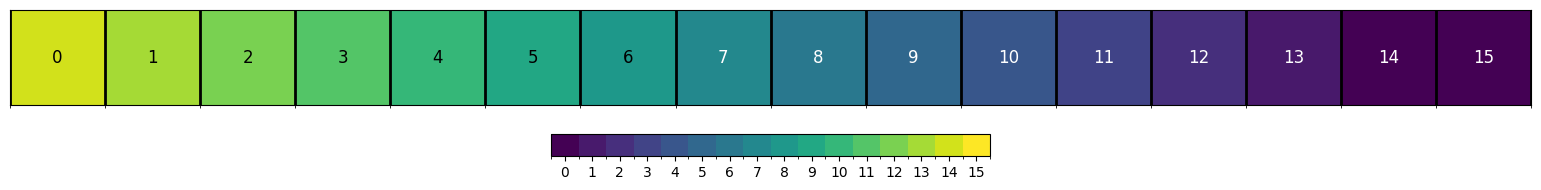}
    \caption{Visualization of the optimal policy. Each cell represents a state, with the inventory level displayed in the center. The color coding indicates the optimal order quantity for that state. For example, at an inventory level of s=1, the optimal action is to order 13 units.}
    \label{fig:inventory_opt_pol}
\end{figure}
We begin with a single instance of our framework. Consider a suboptimal expert whose policy is optimal with respect to a misspecified cost function parameterized by $c_o=5$, $c_h=8$, and $c_s=15$. Intuitively, because this expert perceives a significantly higher holding cost, the resulting policy maintains a reduced inventory level (Figure \ref{fig:inventory_subopt_pol}).
\begin{figure}[ht]
    \centering
    \includegraphics[width=0.8\linewidth]{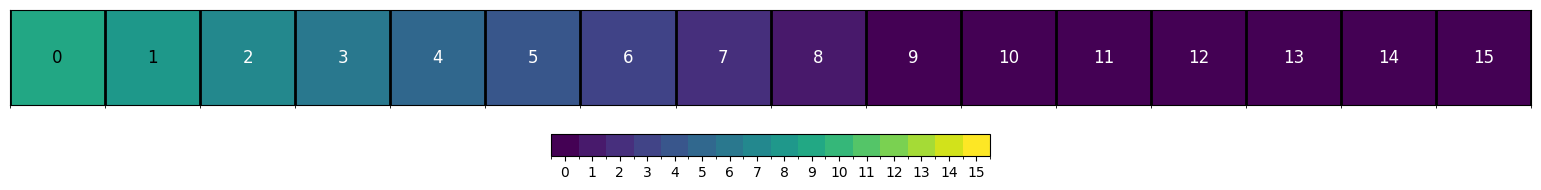}
    \caption{Visualization of the suboptimal policy.}
    \label{fig:inventory_subopt_pol}
\end{figure}
Having defined the suboptimal expert, we next require a proxy cost vector $\hat{\vc}$ representing our prior beliefs on the cost vector. Crucially, this prior must be established ``independently'' of the observed expert to avoid circularity. In this context, independence implies that $\hat{\vc}$ is derived exclusively from exogenous information, such as market prices or physical constraints, rather than being inferred from the expert's trajectories. If $\hat{\vc}$ were derived exclusively from the observed demonstrations, it would fall within the inverse optimality set $\Theta^{\text{inv}}(\vmu_{\pi_E})$, thereby eliminating the necessary trade-off between the expert and the prior (as illustrated in Figure \ref{fig:illustration_IOAL}). Therefore, based strictly on the problem structure and excluding expert data, we suppose $\hat{\vc}$ is defined using the parameters $(c_o,c_h,c_s)$=(4,2,14).
\begin{figure}[ht]
    \centering
    \includegraphics[width=0.8\linewidth]{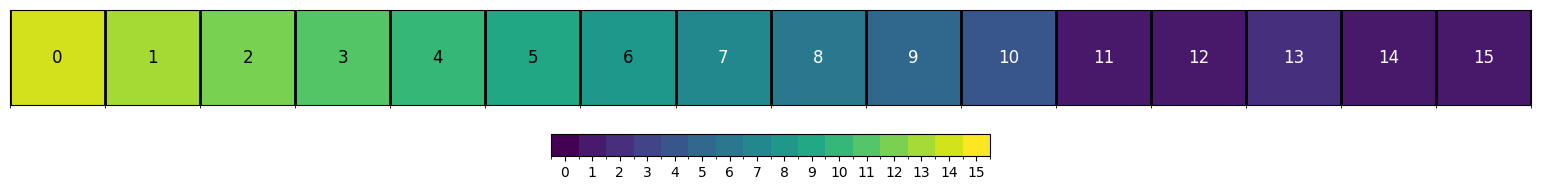}
    \caption{Visualization of the apprentice policy.}
    \label{fig:inventory_apprentice_pol}
\end{figure}
We set $\alpha=0.1$ and execute the SMD algorithm for $N=10^3$ independent runs, with $T=10^4$ iterations and both step-sizes fixed at $10^{-2}$. Figure \ref{fig:inventory_apprentice_pol} illustrates the policy induced by $\vmu^\epsilon$. As observed, the learned policy is identical to the optimal policy with the exception of the last five inventory states, where it orders between one and two units. This demonstrates that the method is capable of recovering a near-optimal policy. To recover the scalar parameters $(c_o, c_h, c_s)$ from the high-dimensional learned cost vector $\vc_A$, we solve a linear least-squares regression problem of the form $\min_{\mathbf{x}} \|\mathbf{A} \mathbf{x} - \mathbf{c}_A\|_2^2$. Here, $\mathbf{x} = [c_o, c_h, c_s]^\top$, and the row of the design matrix $\mathbf{A}$ corresponding to a state-action pair $(s,a)$ is defined by the feature vector $\mathbf{A}_{(s,a)} = [ a,  s+a, -\mathbb{E}[\min(s+a, D)]]$. Solving this regression, we obtain the recovered parameters $(c_o, c_h, c_s) = (3.4, 2.3, 14.1)$. The following sections discuss the effect of the choice of $\hat{\vc}$ and the suboptimality of the expert in more detail.

\subsubsection{Misspecified prior beliefs}
We assess the sensitivity of our framework to misspecifications in the proxy cost vector $\hat{\vc}$. To quantify this, we consider the single-product inventory problem with true parameters $(c_o, c_h, c_s)=(3, 0.5, 15)$. We fix the inputs to the problem \eqref{eq:newproblem} as follows: we use an optimal expert (see Figure \ref{fig:inventory_opt_pol}) and set the regularization parameter $\alpha=0.1$. With these inputs held constant, we perturb the proxy cost vector to analyze the effect of deviations from the ground truth. We introduce noise to the proxy by defining $\hat{c}_o = 3 + \zeta(0.3u_1)$, $\hat{c}_h = 0.5 + \zeta(0.05u_2)$, and $\hat{c}_s = 15 + \zeta(1.5u_3)$, where $u_i \sim U[-1,1]$ are i.i.d.\ random variables. This formulation scales the perturbation such that $\zeta=1$ corresponds to a maximum deviation of $10\%$ per coordinate. We evaluate this sensitivity across a grid of 100 equally spaced values for $\zeta \in [0, 10]$, representing error magnitudes ranging from $0\%$ to $100\%$. For each $\zeta$, we generate 5 independent realizations of the random noise, run the SMD algorithm, and report the average of the recovered costs and the average $\ell_1$ distance between the learned cost vector $\vc^\epsilon$ and the true cost vector $\vc_{\text{true}}$.
\begin{figure}
    \centering
    \subfigure[Recovered cost vectors.]{
        \includegraphics[width=0.45\textwidth]{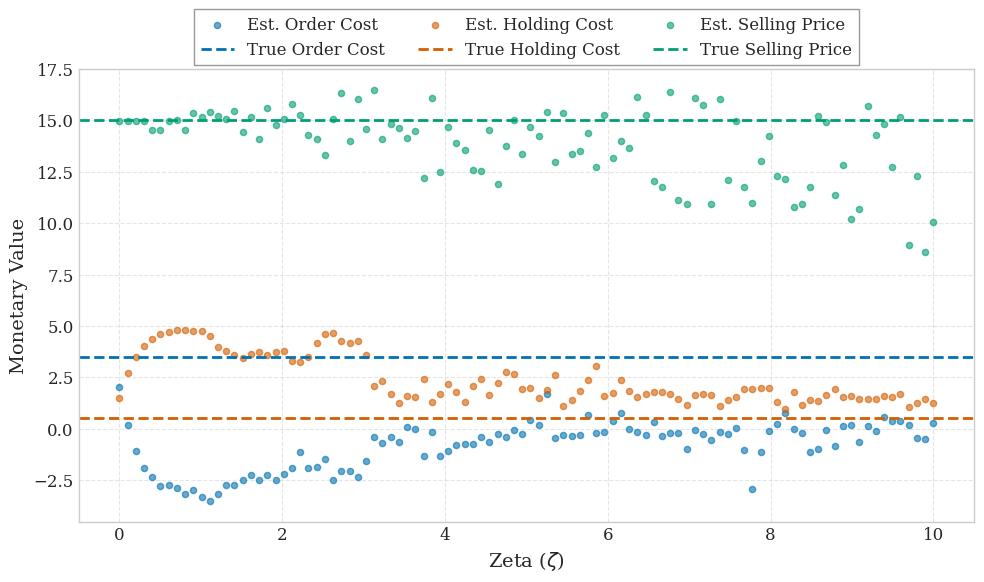}
        \label{fig:inventory_epsilon_true_vs_estimate}
    }
    \hspace{0.5cm}
    \subfigure[$\ell_1$ distance between the learned cost vector and the truth.]{
        \includegraphics[width=0.45\textwidth]{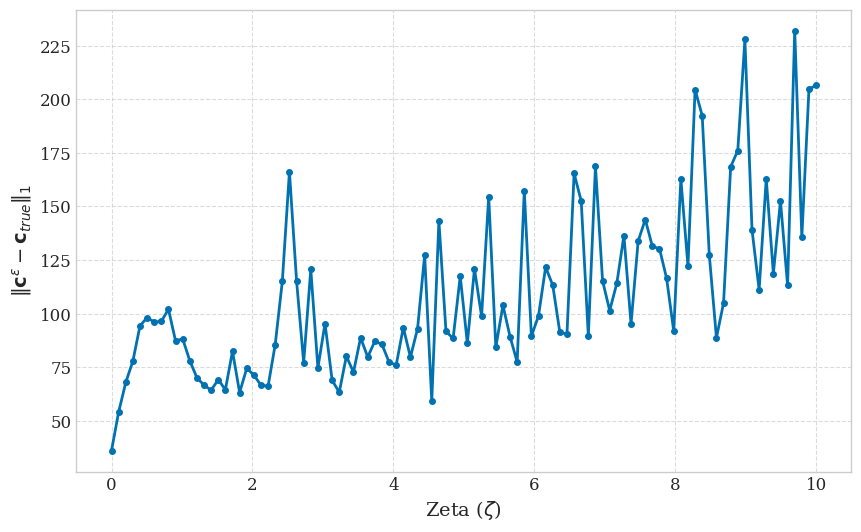}
        \label{fig:inventory_epsilon_l1}
    }
    \caption{Effect of perturbations in the proxy cost vector $\hat{\vc}$ on the learned cost.}
\end{figure}
Figure \ref{fig:inventory_epsilon_true_vs_estimate} shows the difference between the recovered parameters that define the cost vector and the true parameters, while Figure \ref{fig:inventory_epsilon_l1} shows the $\ell_1$-norm between the actual cost vectors. Note that priority is given to the highest cost (selling price), and the distance increases with $\zeta$ as expected. Let us remark that SMD-RLfD is an approximate algorithm; hence, the differences are never null.

\subsubsection{Expert suboptimality and regularization}
Our framework is designed to mitigate potential expert suboptimality by regularizing the solution against a prior belief. Our objective is to demonstrate that if the expert is suboptimal, a correctly specified prior can ``help'' the learning process and recover a performant cost vector and policy. To quantify this impact, we revisit the inventory problem defined by $(c_o, c_h, c_s)=(3, 0.5, 15)$. We establish a controlled setting by fixing a suboptimal expert given by Figure \ref{fig:inventory_subopt_pol} and providing the true cost vector as the proxy $\hat{\vc}$. We then vary the regularization strength $\alpha$ across the interval $[0,0.5]$ and evaluate the learned cost vector, see Figure \ref{fig:inventory_alpha_vs_cost}, and the resulting apprentice policy's performance against the true cost vector, see Figure \ref{fig:inventory_discounted_vs_alpha}. As anticipated, the use of an informative proxy cost vector $\hat{\vc}$ facilitates the recovery of the true cost vector as the regularization parameter $\alpha$ increases. Moreover, the policy induced by the learned $\vmu^\epsilon$ achieves a consistently lower discounted expected cost with respect to the true cost vector $\vc_{\text{true}}$ across the tested values of $\alpha$. This indicates that the learned policy is robust to variations in $\alpha$ and outperforms the expert. Note that the discounted expected cost for the policy induced by $\pi_{\vmu^\epsilon}$ was computed using policy evaluation.

\begin{figure}
    \centering
    \subfigure[Recovered cost vectors.]{
        \includegraphics[width=0.45\textwidth]{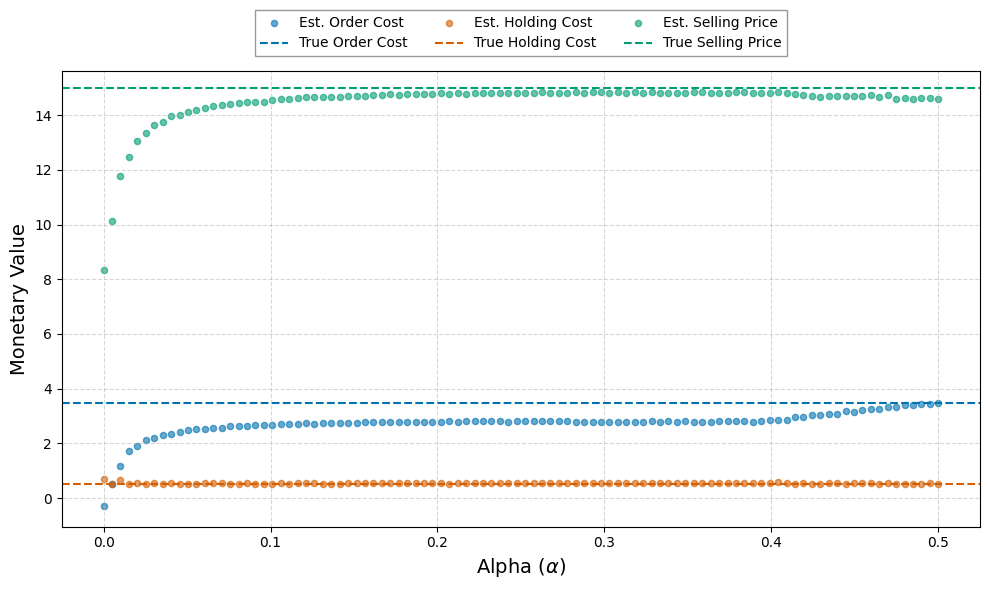}
        \label{fig:inventory_alpha_vs_cost}
    }
    \hspace{0.5cm}
    \subfigure[Discounted expected cost.]{
        \includegraphics[width=0.45\textwidth]{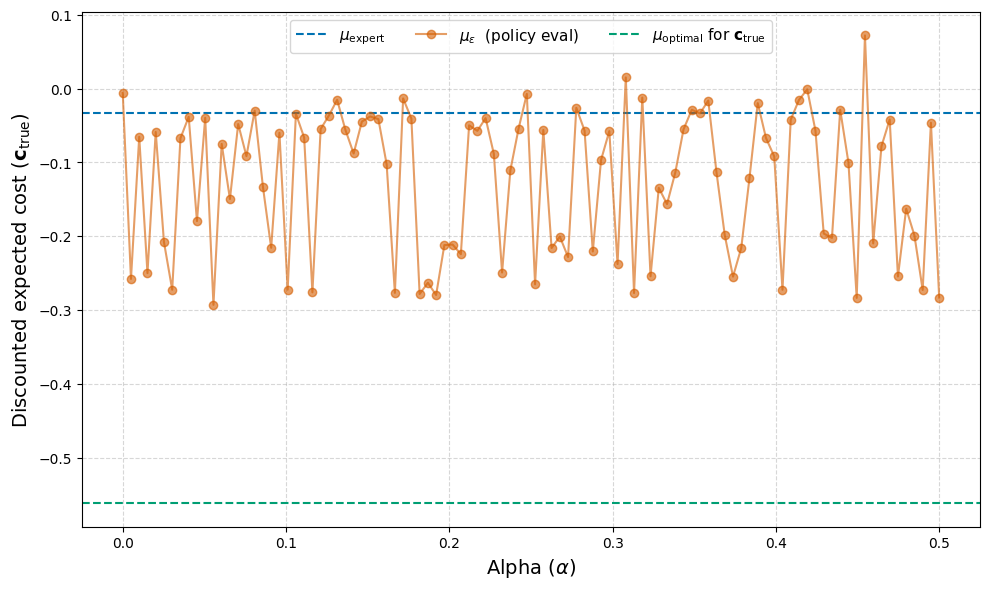}
        \label{fig:inventory_discounted_vs_alpha}
    }
    \caption{Learned cost vectors and apprentice policies under varying regularization levels $\alpha$.}
\end{figure}

Using the same experimental setting, we analyze the quantity $\alpha \Vert \hat{\vc} - \vc^\epsilon \Vert + \rho_{\vc^\epsilon}(\pi_E) - \rho_{\vc^\epsilon}(\pi_A)$. As established in Proposition~\ref{prop:epsilonapproxOptimality}, this expression links the outputs of SMD-RLfD to the optimal values of~\ref{eq:IO-AL}, which are characterized in Proposition~\ref{prop:suboptimalexpert}. Figure~\ref{fig:inventory_alpha_proposition} illustrates the trade-off given by the proposition. The left vertical axis displays the discounted expected cost with respect to the learned cost vector $\vc^\epsilon$ for three policies: the expert policy $\vmu_E$, the policy induced by $\vmu^\epsilon$, and the optimal policy for $\vc^\epsilon$. The right vertical axis plots the norm of the difference $\Vert \hat{\vc} - \vc^\epsilon \Vert$ across varying values of $\alpha$.
\begin{figure}
    \centering
    \includegraphics[width=0.5\linewidth]{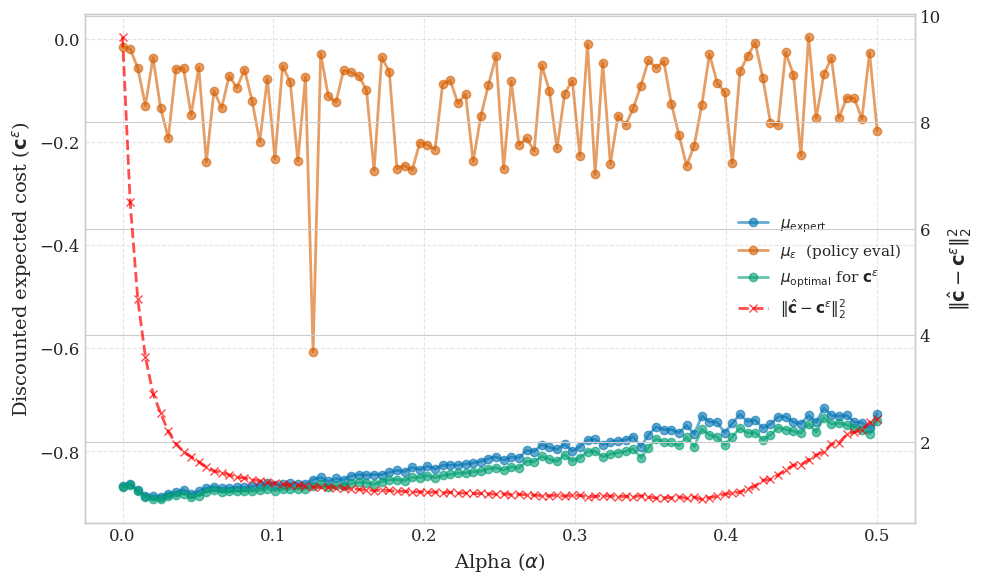}
    \caption{Illustration of Propositions \ref{prop:suboptimalexpert} and \ref{prop:epsilonapproxOptimality}.}
    \label{fig:inventory_alpha_proposition}
\end{figure}

We observe a clear balance between these terms: as $\alpha$ increases, the optimization places greater weight on the norm, reducing $\Vert \hat{\vc} - \vc^\epsilon \Vert$ while allowing the performance gap $\rho_{\vc^\epsilon}(\pi_E) - \rho_{\vc^\epsilon}(\pi_A)$ to widen. Conversely, when $\alpha$ is small, the performance gap becomes negligible, while the norm difference increases inversely. In any case, note that $\vc^\epsilon$ is chosen such that the expert policy is nearly optimal. Finally, we observe that the discounted expected cost of the policy induced by $\vmu^\epsilon$ remains stable to variations in $\alpha$.

\subsubsection{Restricting the search space to a convex hull}
As it was noted in Section \ref{sec:Lfd-AL}, the apprenticeship learning literature typically assumes that the true cost vector $\vc_{\text{true}}$ lies within the convex hull of a prespecified set of basis vectors $\{\vc_i\}_{i=1}^{n_c}$ \citep{Syed2008, Kamoutsi2021}. While this assumption reduces the dimensionality of the search space from $|\St||\Ac|$ to $n_c$, it limits flexibility and requires a robust method for feature engineering. Since the algorithm presented in \cite{Kamoutsi2021} targets a linearly relaxed saddle-point formulation with different convergence properties, we adapt SMD-RLfD to optimize over the vectors $\vc_\vw = \mC\vw$, where $\vw\in\Delta^{n_c}$ and the columns of $\mC \in \R^{|\St||\Ac|\times n_c}$ correspond to the basis vectors. This adaptation entails updating the gradient estimators for the objective function \eqref{eq:newproblem} and shifting the mirror descent step from the box $\sB_1^{|\St||\Ac|}$ to the simplex $\Delta^{n_c}$. Moreover, setting $\alpha=0$ in this adapted formulation recovers the problem solved by \cite{Kamoutsi2021}, as guaranteed by the equality in Theorem \ref{theorem:IO-IRL}. This enables a controlled comparison of the optimization domains by applying both the original and modified algorithms to the equivalent unregularized problem where $\alpha=0$.

The inventory problem serves as an ideal testbed for the convex hull formulation because its cost vector is a linear combination of three features: ordering $\vc^o$, holding $\vc^h$, and sales $\vc^s$. This structure effectively reduces the dimensionality of the cost parameters from $|\St| |\Ac|$ to $3$. We compare the convergence of SMD-RLfD against the modified variant (SMD-H) that optimizes weights $\vw \in \Delta^3$, when considering an optimal expert and maintaining common algorithm parameters $T=10,000$ and $N=100$. Figure \ref{fig:inventory_convergence_comparison} displays the $\ell_1$-norm of the difference between solutions at consecutive iterations. The results indicate that the convex hull approach accelerates the convergence of $\vw$, a behavior consistent with the suitability of mirror descent for optimization over a simplex domain.
\begin{figure}
    \centering
    \includegraphics[width=0.5\linewidth]{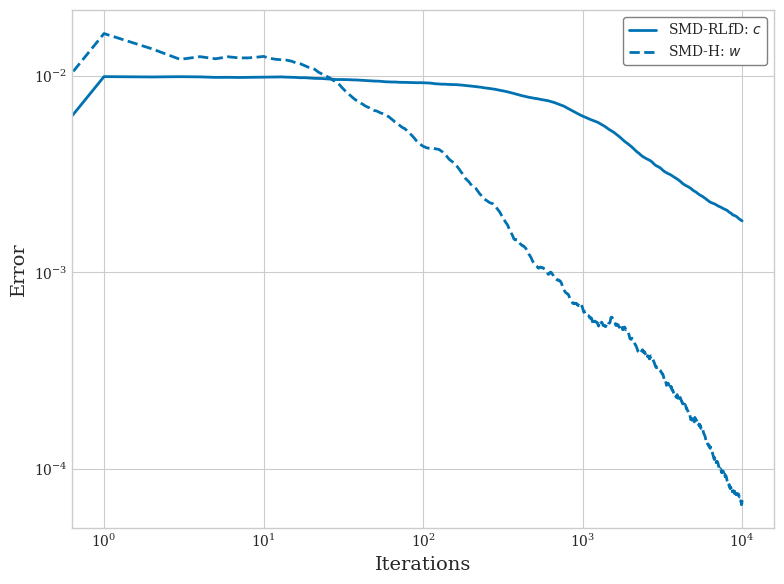}
    \caption{Impact of the search space geometry (convex hull vs. box) on the convergence rate of the cost vector.}
    \label{fig:inventory_convergence_comparison}
\end{figure}

We extend the experimental evaluation to varying state space dimensions. For each configuration, we consider an optimal expert policy using the fixed parameters $(c_o, c_h, c_s) = (3, 0.5, 15)$. Figure \ref{fig:inventory_apprentice_comparison} presents the discounted expected cost of the resulting apprentice policies with respect to the true cost vector. We observe that while the convex hull formulation yields similar performance in smaller state spaces, our framework outperforms it as the problem dimension grows. This trend may be attributed to the stochastic nature of the solution algorithms. The box formulation exhibits greater flexibility in adapting to error variations, whereas the convex hull proves too rigid to adapt effectively in higher dimensions.
\begin{figure}[ht]
    \centering
    \includegraphics[width=0.5\linewidth]{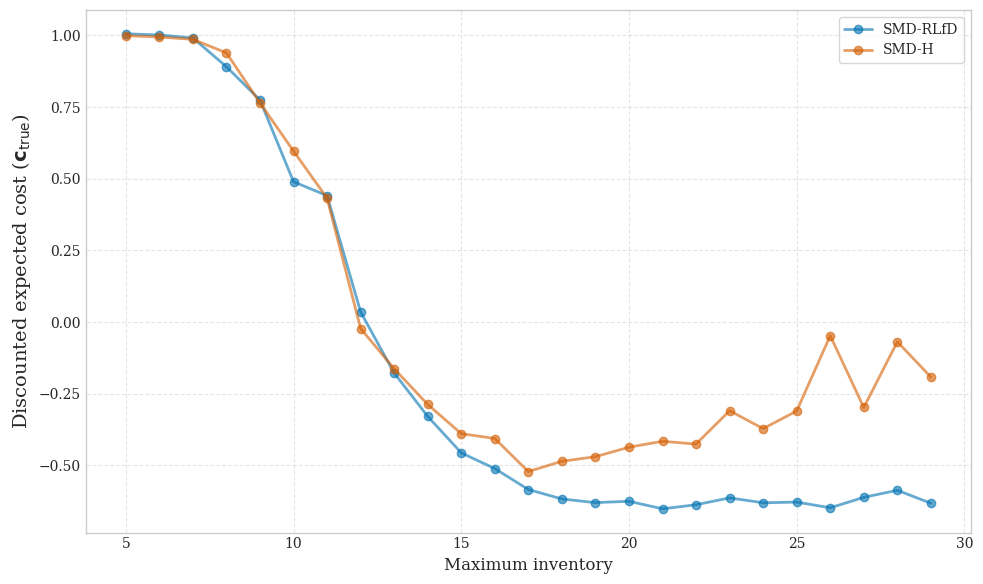}
    \caption{Comparison of the discounted expected cost between the proposed framework and the convex hull approach across state space sizes.}
    \label{fig:inventory_apprentice_comparison}
\end{figure}

\subsection{Gridworld} \label{sec:gridworld}
We use a standard $H \times W$ Gridworld environment (Figure \ref{fig:gridworld}), where each cell is a unique state. Obstacles (shown in red) incur a cost of $1$, terminal cells (shown in green) incur a cost of $-1$, and all other cells (shown in white) incur a cost of $0$. The action set is $\{\text{up}, \text{down}, \text{left}, \text{right}\}$, but a ``wind'' introduces a $20\%$ chance of drifting left or otherwise altering the intended move. If the resulting move is out of bounds, the agent remains in the same cell. The discount factor is $0.7$, and initial states are chosen uniformly. We solve the corresponding \eqref{eq:MDP-P} with a linear solver to obtain the optimal occupancy measure (Figure \ref{fig:optimal-policy}). To construct a suboptimal expert, we terminate the solver early and use the resulting $\vmu$ as the expert's occupancy measure (Figure \ref{fig:expert-policy}). To visualize the policies induced by these occupancy measures, we extract the most likely action at each state by computing $\argmax_{a\in \Ac} \vmu(s,a)$ and display it in the corresponding cell.
\begin{figure}[ht]
    \centering
    \subfigure[Gridworld environment.]{
        \includegraphics[width=0.25\textwidth]{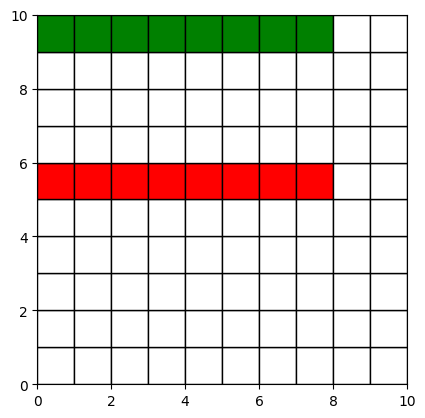}
        \label{fig:gridworld}
    }
    \hspace{0.5cm}
    \subfigure[Optimal policy.]{
        \includegraphics[width=0.25\textwidth]{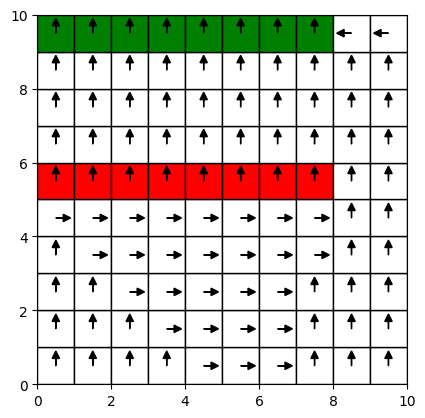}
        \label{fig:optimal-policy}
    }
    \hspace{0.5cm}
    \subfigure[Expert's policy.]{
        \includegraphics[width=0.25\textwidth]{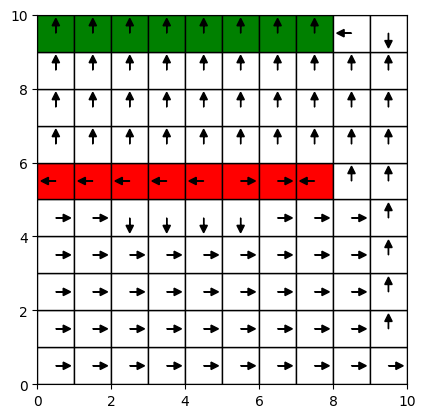}
        \label{fig:expert-policy}
    }
    \caption{Illustration of the Gridworld environment, the optimal policy, and the expert's policy.}
\end{figure}
Unlike the inventory control example, defining the cost vector in this setting as a convex combination of basis vectors is not straightforward. This highlights a key advantage of our framework's flexibility: it obviates the need for preliminary feature engineering or basis estimation processes that are often time-consuming and prone to error. Consequently, our method is directly applicable without requiring these preprocessing steps, since even without a prior belief of the cost vector, we can just set $\alpha=0$. Furthermore, modeling this region as a convex hull would require $2^{|\St||\Ac|}$  basis vectors, rendering the approach computationally intractable, even for small state-action spaces.

\subsubsection{Impact of regularization on the learned cost vector}
We study the effect of the regularization term $\alpha\Vert\mathbf{c} - \hat{\mathbf{c}}\Vert^2_2$ on problem \eqref{eq:newproblem} to demonstrate that incorporating prior knowledge, even when imperfect, yields learned cost vectors that align more closely with the true environment. To this end, we define a cost vector $\hat{\vc}$ that is zero everywhere except for a randomly selected subset of obstacle and goal states. This choice was made based on practical considerations, as in most real-world scenarios, we rarely have access to a highly accurate estimate of the cost vector. Specifically, for obstacles, $\hat{\vc}$ is set to $1$ for a randomly selected $50\%$ of the obstacle states, while for goal states, $\hat{\mathbf{c}}$ is set to $-1$ for a randomly chosen $50\%$ of them. Moreover, regarding the algorithm's parameters, we chose $N=20$, $T=10^{5}$, and both step-sizes as $10^{-2}$. \par

Figure \ref{fig:effect-reg} depicts the learned cost vectors for each action $\{\text{up}, \text{down}, \text{left}, \text{right}\}$ under varying levels of regularization $\alpha$. As $\alpha$ increases, the cost vectors display more white regions, indicating near-zero cost values, and more accurately highlight the main obstacles. In turn, this leads to a better approximation of the true cost structure. However, when the regularization is too strong, it may overly penalize costs associated with obstacle areas that are less frequently demonstrated, thereby ignoring parts of the environment that do not appear in the demonstration data. Consequently, selecting an appropriate $\alpha$ is crucial to achieve a cost vector that balances fidelity to the true environment and robustness in identifying cost structures that are not present in the estimate $\hat{\vc}$.
\begin{figure}
    \centering
    % Use our new 'M' column type. 
    % I slightly reduced 0.2 to 0.19 to ensure it fits within page margins with padding.
    \begin{tabular}{ M{0.75cm} M{0.19\textwidth} M{0.19\textwidth} M{0.19\textwidth} M{0.19\textwidth} }
    
         & $\alpha=0.1$ & $\alpha=0.005$ & $\alpha=0.001$ & $\alpha=0$ \\
         
         \text{up} 
            & \includegraphics[width=\linewidth]{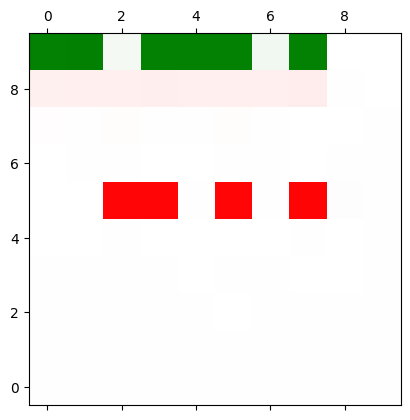}
            & \includegraphics[width=\linewidth]{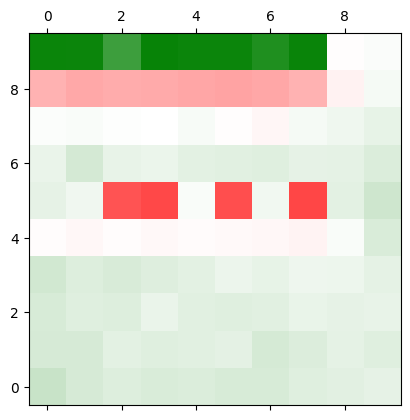}
            & \includegraphics[width=\linewidth]{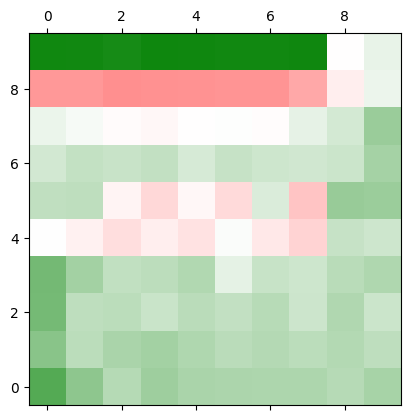}
            & \includegraphics[width=\linewidth]{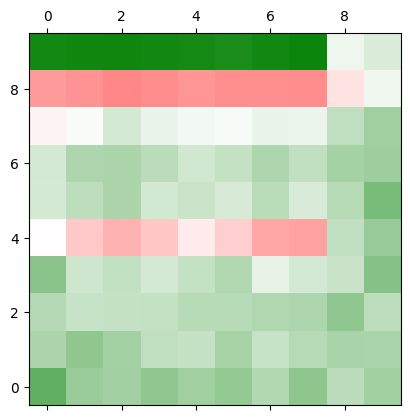} \\
            
         \text{down}
            & \includegraphics[width=\linewidth]{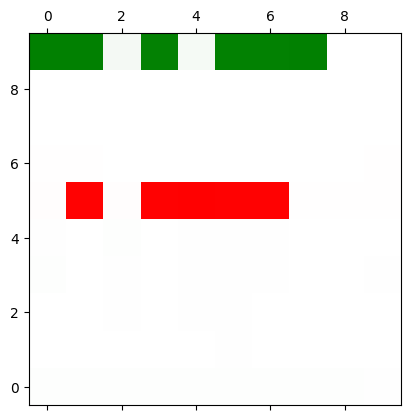}
            & \includegraphics[width=\linewidth]{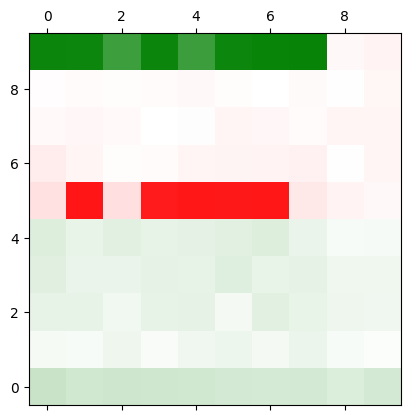}
            & \includegraphics[width=\linewidth]{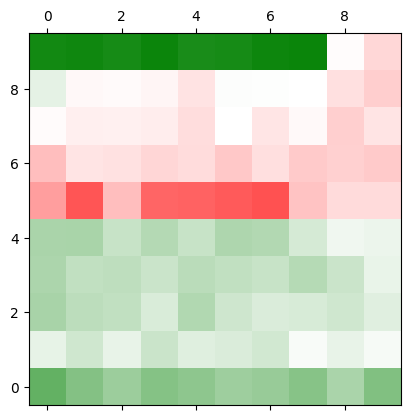}
            & \includegraphics[width=\linewidth]{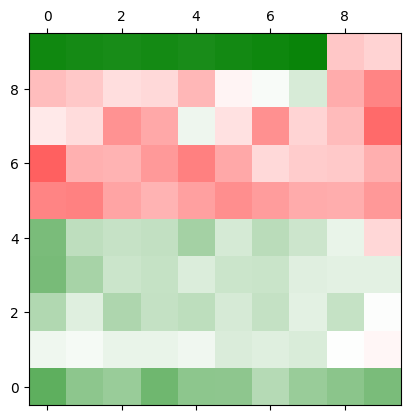}\\
            
        \text{left}
            & \includegraphics[width=\linewidth]{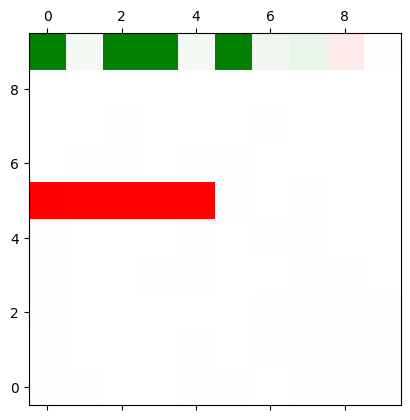}
            & \includegraphics[width=\linewidth]{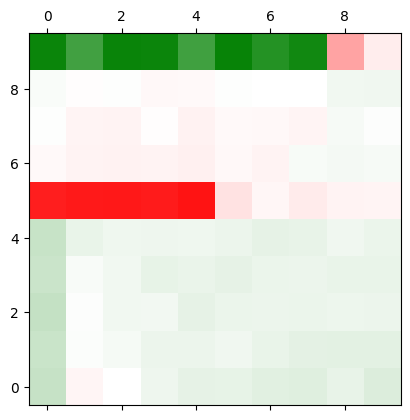}
            & \includegraphics[width=\linewidth]{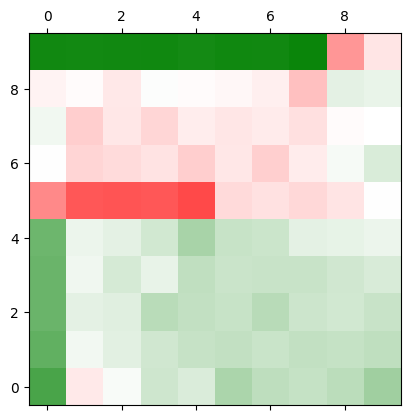}
            & \includegraphics[width=\linewidth]{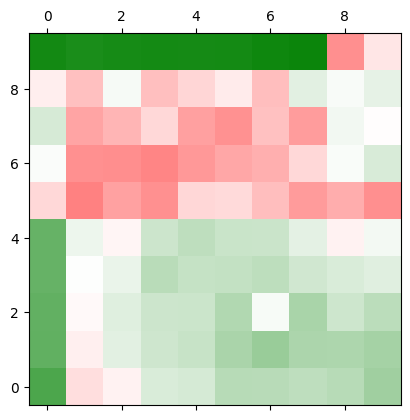} \\
            
        \text{right}
            & \includegraphics[width=\linewidth]{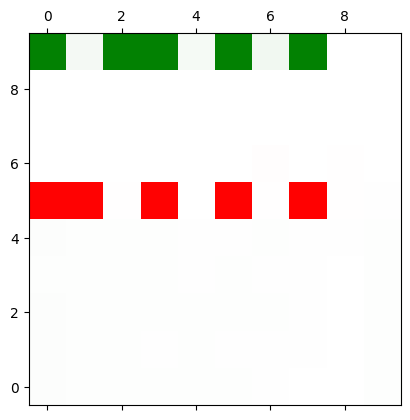}
            & \includegraphics[width=\linewidth]{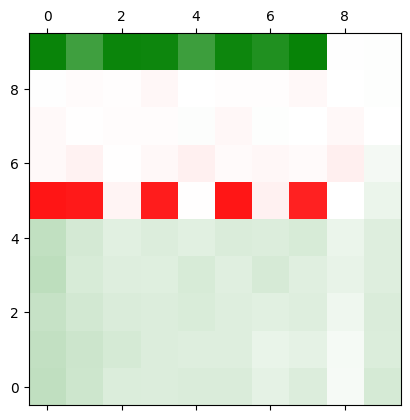}
            & \includegraphics[width=\linewidth]{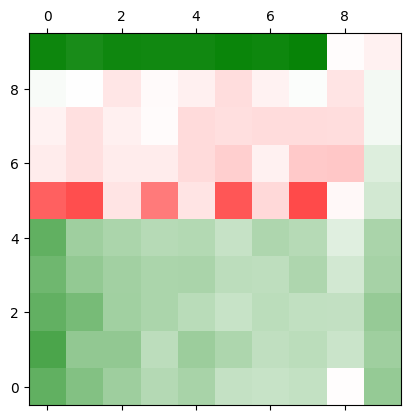}
            & \includegraphics[width=\linewidth]{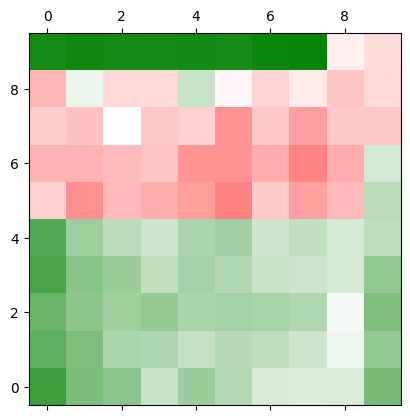}
            
    \end{tabular}
    \caption{Effect of the regularization on the cost vector.}
    \label{fig:effect-reg}
\end{figure}

Figure \ref{fig:effect-reg-pol} displays the apprentice policy obtained using SMD-RLfD for various values of $\alpha$. Notably, when $\alpha$ is set to 0.005 or 0.001, the apprentice policy closely resembles the one derived by exclusively weighting the expert’s policy (i.e., $\alpha=0$). This observation is significant, as the previous analysis demonstrated that incorporating regularization yields a cost vector for the apprentice policies that aligns more closely with the true cost vector.
\begin{figure}
    \centering
    % 5 columns: one for the row label + 4 columns for the images
    \begin{tabular}{ccccc}
      % First row: empty cell + 4 column headers
      & $\alpha=0.1$ & $\alpha=0.005$  & $\alpha=0.001$  & $\alpha=0$ \\
      % Second row: row label + 4 images
         & \includegraphics[width=0.2\textwidth]{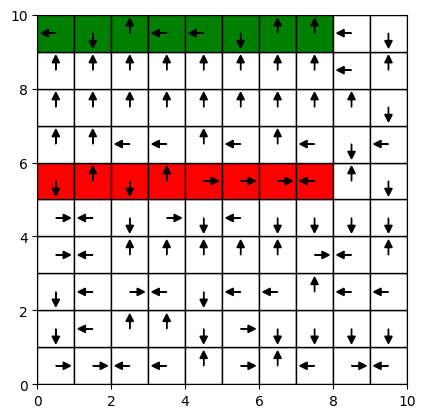}
         & \includegraphics[width=0.2\textwidth]{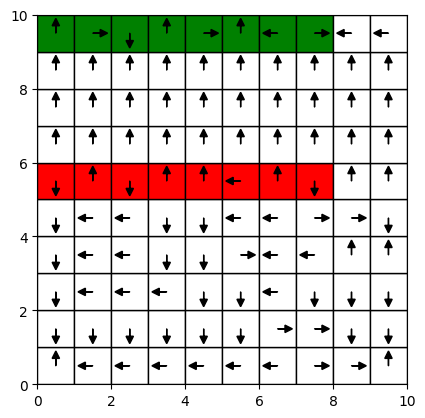}
         & \includegraphics[width=0.2\textwidth]{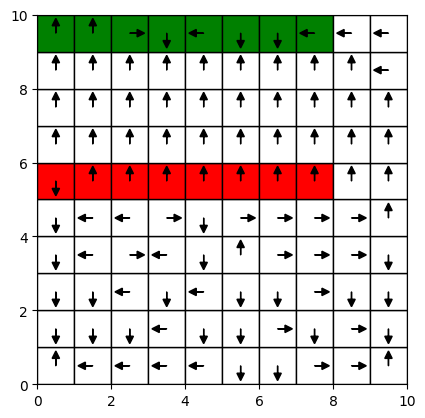}
         & \includegraphics[width=0.2\textwidth]{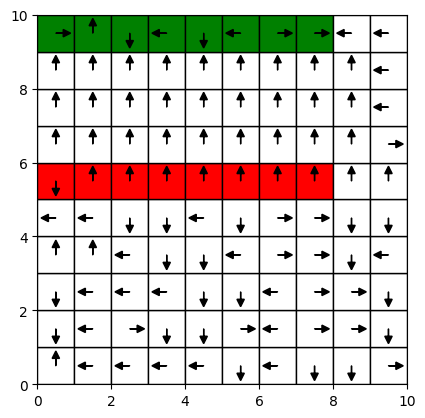} \\
    \end{tabular}
    \caption{Effect of regularization on the apprentice policy.}
    \label{fig:effect-reg-pol}
\end{figure}

\subsubsection{Convergence and regularization}
In the same experimental setting, we examined the convergence of the solutions up to iteration $t<T$ by computing the $\ell_1$-norm of the difference between the solutions at iterations $t$ and $t-1$. We plotted this norm for each of the $\alpha$ values considered in the previous experiments. According to this sense of convergence, all $\alpha$ values exhibit comparable convergence rates for $\vu$ and $\vmu$. However, for $\vc$, the convergence rate at $\alpha=0.1$ is faster than that observed for smaller values of $\alpha$. This behavior aligns with the increased penalty imposed for deviating from $\hat{\vc}$ under stronger regularization, thereby accelerating convergence for $\vc$.

\begin{figure}[ht]
    \centering
    % 5 columns: one for the row label + 4 columns for the images
    \begin{tabular}{ccccc}
      % First row: empty cell + 4 column headers
      & $\alpha=0.1$ & $\alpha=0.005$  & $\alpha=0.001$  & $\alpha=0$ \\
      % Second row: row label + 4 images
         & \includegraphics[width=0.215\textwidth]{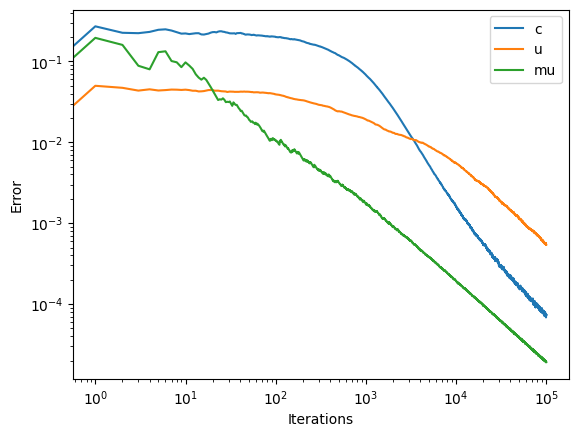}
         & \includegraphics[width=0.215\textwidth]{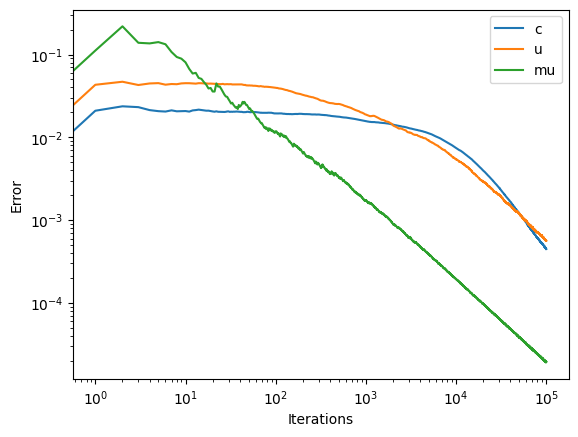}
         & \includegraphics[width=0.215\textwidth]{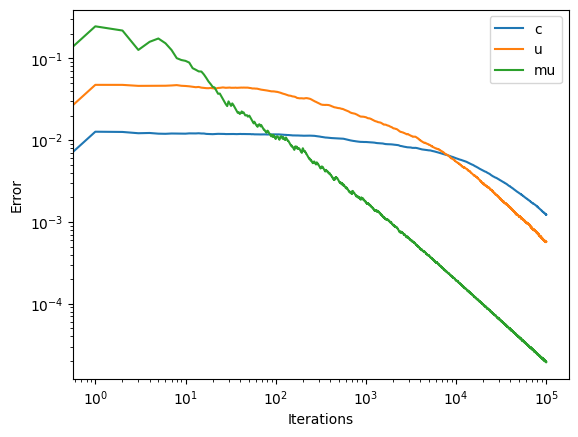}
         & \includegraphics[width=0.215\textwidth]{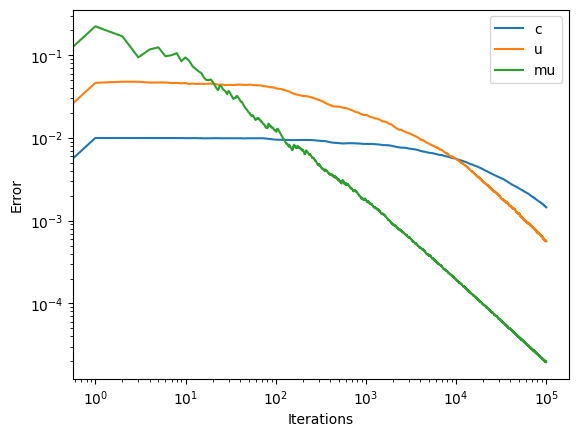} \\
    \end{tabular}
    \caption{Effect of regularization on the difference of solutions found between iterations.}
    \label{fig:effect-reg-conv}
\end{figure}

However, it is important to note that Theorem \ref{theorem:algconvergence} characterizes convergence in terms of the duality gap. In Figure \ref{fig:effect-reg-dualitygap}, we compute the duality gap of the solution every 25 iterations, using IPOPT \citep{Wachter2006}. The algorithm parameters remain unchanged, except that, due to computational constraints, we reduce the grid size to $6\times6$ and limit the execution to $T=10^4$ iterations. As expected from the results of Theorem \ref{theorem:algconvergence}, stronger regularization leads to slower convergence of the duality gap, a trend illustrated in this figure. While we did not use the theoretically prescribed step sizes due to their small magnitude, the figure still offers valuable insights into how regularization affects the convergence behavior.
\begin{figure}[ht]
    \centering
    % 5 columns: one for the row label + 4 columns for the images
    \begin{tabular}{cccc}
      % First row: empty cell + 4 column headers
      & $\alpha=0.25$  & $\alpha=0.1$  & $\alpha=0$ \\
      % Second row: row label + 4 images
         %& \includegraphics[width=0.25\textwidth]{figures/gap0.9.png}
         & \includegraphics[width=0.28\textwidth]{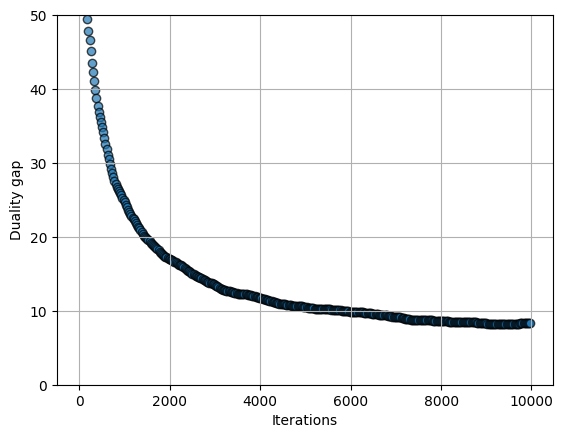}
         & \includegraphics[width=0.28\textwidth]{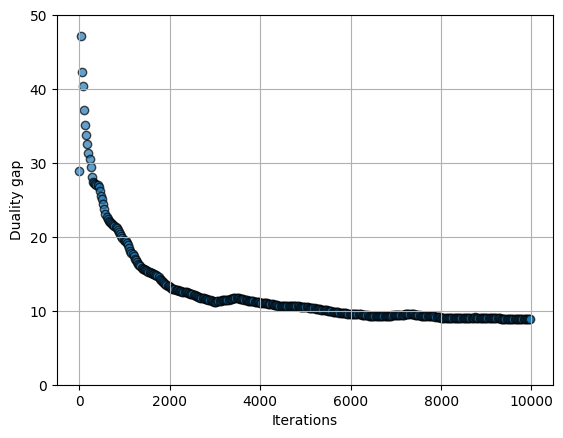}
         & \includegraphics[width=0.28\textwidth]{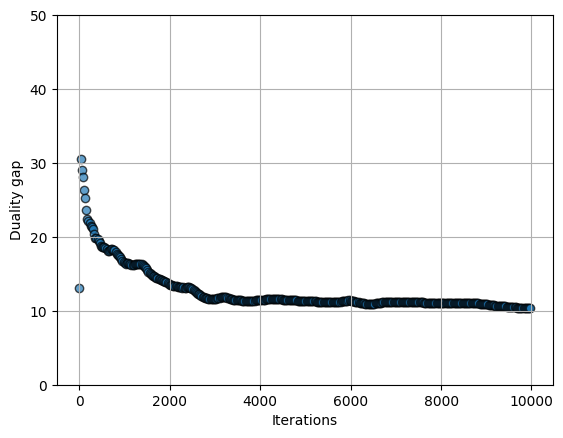} \\
    \end{tabular}
    \caption{Duality gap convergence.}
    \label{fig:effect-reg-dualitygap}
\end{figure}

\section{Conclusion}
This work establishes a unified perspective on IRL and AL by grounding them in the IO framework for MDPs. A key theoretical contribution is the demonstration that the convex-analytic AL formalism proposed by \cite{Kamoutsi2021} emerges as a specific relaxation of this broader framework. Central to the proposed methodology is the incorporation of prior beliefs on the cost function to address the inherent ill-posedness of the IRL problem, where multiple cost functions can explain the same expert behavior. By introducing a proxy cost vector $\hat{\vc}$ and a regularization parameter $\alpha$, we formulate the \eqref{eq:IO-AL} problem tailored for the scenarios in which the expert is suboptimal. This approach is designed to mitigate the impact of expert suboptimality, with the regularization term serving as a mechanism to modulate the trade-off between fidelity to the observed demonstrations and adherence to the informative prior.

To solve the optimization problem posed in \eqref{eq:IO-AL}, we leverage Lagrangian duality to recast the formulation as a convex-concave-min-max problem suitable for stochastic mirror descent methods. We introduce the SMD-RLfD algorithm, an adaptation of the method proposed by \cite{Jin2020} to solve MDPs, and derive the necessary gradient estimators for this framework. Furthermore, we establish convergence guarantees, explicitly characterizing the relationship between the expected $\epsilon$-approximate solutions generated by SMD-RLfD and the optimal solution of the original formulation \eqref{eq:IO-AL}.

Numerical experiments underscore the critical role of regularization on learning performance. Our ablation analysis reveals that in scenarios involving a suboptimal expert and an informative proxy, increasing the regularization parameter $\alpha$ significantly improves the alignment between the recovered cost vectors and the ground truth. The resulting induced policy is notably robust to variations in $\alpha$ and consistently outperforms the expert. In comparison to the convex hull approach of \citet{Kamoutsi2021}, we find that while the simplex-based formulation may offer faster convergence, our framework yields induced policies with superior performance as the state space expands. Furthermore, experiments on higher-dimensional Gridworld environments highlight the importance of regularization in learning a cost vector that correctly approximates the true cost vector and corroborate that stronger regularization accelerates the convergence of the cost vector; however, consistent with Theorem~\ref{theorem:algconvergence}, this comes at the expense of slower convergence for the duality gap.

Future research could extend our methodology in several meaningful directions. First, establishing a criterion for selecting the regularization parameter $\alpha$ remains an open challenge. Second, in scenarios where the cost vector is known to be sparse (see the Gridworld example in Section \ref{sec:gridworld}), it would be useful to analyze problem \eqref{eq:IO-AL} with the $\ell_0$ norm in the objective, $\Vert \vc\Vert_0$, to explicitly enforce sparsity. Finally, further work is warranted to assess the scalability of this method to large-scale, high-dimensional problems and to develop robust techniques for constructing the proxy cost vector in complex real-world scenarios.

\textbf{Acknowledgements}

We thank Mateo Campo for his helpful insights. M. Junca was supported by the Research Fund of Facultad de Ciencias, Universidad de Los Andes INV-2025-213-3470.

\bibliography{main}
\bibliographystyle{tmlr}

\appendix
\section{Appendix - Some proofs}
\subsection{MDPs}
\begin{lemma} \label{lemma:boundValueFunction}
    Given a MDP with discount factor $\gamma\in(0,1)$, then the value function satisfies $$\evV^\pi_\vc(s) \leq \Vert \vc \Vert_\infty = 1$$  for any policy $\pi\in\Pi_0$ and any state $s\in\St$.
\end{lemma}
\begin{proof}
    Given a policy $\pi\in\Pi_0$ and a state $s\in\St$ the value function
    \begin{align*}
        \evV^\pi_\vc (s) &=  (1-\gamma)\E^\pi_{s}\left[ \sum\limits_{t=0}^\infty \gamma^t \evc(s_t,a_t)\right]\\
        &\leq (1-\gamma)\E^\pi_s \left[ \sum\limits_{t=0}^\infty \gamma^t \Vert \vc \Vert_\infty\right]\\
        &= (1-\gamma)\frac{1}{(1-\gamma)}\Vert\vc\Vert_\infty\\
        &= 1.
    \end{align*}
\end{proof}

\subsection{Gradient estimation}
\begin{proof}[Proof of Lemma \ref{lemma:cu_grad}]
        The unbiasedness follows from the following observations
        \begin{align*}
            &\sum\limits_{s,a}\frac{|\St||\Ac|}{|\St||\Ac|}\cdot 2\alpha\left(\evc_t(s,a)\ve^{(s,a)}-\hat{\evc}(s,a)\ve^{(s,a)}\right)\\
            &+\sum\limits_{s_E,a_E} \evmu_{\pi_E}(s_E,a_E)\ve^{(s_E,a_E)} - \sum\limits_{s_t,a_t}\evmu_t(s_t,a_t)\ve^{(s_t,a_t)}\\
            &= 2\alpha(\vc_t - \hat{\vc}) + \vmu_{\pi_E} - \vmu_t
        \end{align*}
        and
        \begin{align*}
            \frac{1}{(1-\gamma)} &\sum\limits_{s_t',a_t,s_t} \evmu_t(s_t,a_t)P(s^\prime_t\:|\:s_t,a_t)(\ve^{(s_t)} - \gamma \ve^{(s_t^\prime)})\\
            &= \frac{1}{(1-\gamma)}\left(\sum\limits_{s_t,a_t}\evmu_t(s_t,a_t)\ve^{(s_t)} -\gamma \sum\limits_{s_t',a_t,s_t} \evmu_t(s_t,a_t)P(s^\prime_t\:|\:s_t,a_t) \ve^{(s_t^\prime)} \right)\\
            &= \frac{1}{(1-\gamma)}(\mB\vmu - \gamma\mP\vmu)\\
            &= \mTg\vmu.
        \end{align*}
        Thus, we obtain
        \begin{align*}
        \E\left[ \Tilde{g}^{(\vc,\vu)}((\vc,\vu),\vmu) \right] &=                                                 \begin{pmatrix}
                    2\alpha(\vc_t-\hat{\vc}) + \vmu_{\pi_E} - \vmu_t \\
                    \mTg\vmu - \mTg\vmu_{\pi_E}
                    \end{pmatrix}.
        \end{align*}
        To prove the bound on the second-moment, note that
        \begin{align*}
            &\E\left[\Big\Vert |\St||\Ac|\cdot 2\alpha\left(\evc_t(s,a)\ve^{(s,a)}-\hat{\evc}(s,a)\ve^{(s,a)}\right) + \ve^{(s_E,a_E)} - \ve^{(s_t,a_t)}\Big\Vert_2^2\right] \\
            &\leq 2\E\left[4|\St||\Ac|\alpha^2\Big\Vert\evc_t(s,a)\ve^{(s,a)}-\hat{\evc}(s,a)\ve^{(s,a)} \Big\Vert_2^2 + \Big\Vert \ve^{(s_E,a_E)} - \ve^{(s_t,a_t)}\Big\Vert_2^2\right]\\
            &\leq 2\E\left[4|\St||\Ac|\alpha^2\cdot4 + 2\right]\\
            &= 32|\St||\Ac|\alpha^2 + 4,
        \end{align*}
        and 
        \begin{align*}
            \E\left[\Big\Vert \frac{1}{(1-\gamma)}\left( \ve^{(s_t)} - \gamma \ve^{(s_t^\prime)} - (\ve^{(s_E)} - \gamma\ve^{(s_E^\prime)})\right) \Big\Vert_2^2\right] &\leq \E\left[\frac{2(1+\gamma^2)}{(1-\gamma)^2}\right]\\
            &= \frac{2(1+\gamma^2)}{(1-\gamma)^2}.
        \end{align*}
        Hence, we can provide the bound
        \begin{align*}
            \E\left[ \big\Vert \Tilde{g}^{(\vc,\vu)}((\vc,\vu),\vmu) \big\Vert^2_2 \right]
            &\stackrel{(i)}{\leq} 2\left[ 32|\St||\Ac|\alpha^2 + 4 + \frac{2(1+\gamma^2)}{(1-\gamma)^2}\right] \\
            &= 64\alpha^2\cdot |\St||\Ac| + \frac{4(1+\gamma^2)}{(1-\gamma)^2} + 8.
        \end{align*}
        where in $(i)$ we used $\Vert \vx+ \vy\Vert^2 \leq 2[\Vert \vx \Vert^2 + \Vert \vy \Vert^2]$.
     \end{proof}

\begin{proof}[Proof of Lemma \ref{lemma:mu_grad}]
    The unbiasedness follows from
    \begin{align*}
        \E[\Tilde{g}^{\vmu}((\vc_t,\vu_t),\vmu_t)] &= \sum\limits_{s,a,s'}\frac{1}{|\St||\Ac|}P(s'\:|\:s,a)\cdot|\St||\Ac|\left(\evc_t(s,a)\ve^{(s,a)} - \frac{1}{(1-\gamma)}(\evu_t(s)\ve^{(s,a)}-\gamma\evu_t(s')\ve^{(s,a)})\right)\\
        &= \sum\limits_{s,a}\evc_t(s,a)\ve^{(s,a)} - \frac{1}{(1-\gamma)}\left(\sum\limits_{s,a}\evu_t(s)\ve^{(s,a)}-\gamma\sum\limits_{s,a}\left(\sum\limits_{s'}P(s'\:|\:s,a)\evu_t(s')\right)\ve^{(s,a)}\right)\\
        &= \vc_t - \frac{1}{(1-\gamma)}\left(\mB^\top\vu - \gamma\mP^\top\vu\right)\\
        &= \vc_t - \mTg^\top \vu.
    \end{align*}
    For the bound on the maximum entry observe that
    \begin{align*}
        \Vert \Tilde{g}^{\vmu} ((\vc_t,\vu_t),\vmu_t) \Vert_\infty &= |\St||\Ac|\max_{s,a,s'}\left\{\bigg\vert \evc_t(s,a) - \frac{\evu_t(s) - \gamma \evu_t(s')}{(1-\gamma)} \bigg\vert\right\}\\
        &\leq |\St||\Ac|\left(\Vert \vc_t \Vert_\infty + \Vert\vu_t\Vert_\infty\frac{(1+\gamma)}{(1-\gamma)}\right)\\
        &= \frac{2|\St||\Ac|}{(1-\gamma)}
    \end{align*}
    Finally, the bound on the second-moment can be obtained by:
    \begin{align*}
        \E\left[ \Vert \Tilde{g}^{(\vmu)}((\vc_t,\vu_t),\vmu_t) \Vert^2_{\vmu'} \right] &\stackrel{(i)}{\leq} |\St|^2|\Ac|^2\cdot\E \left[ 2\Vert \evc_t(s,a)\ve^{(s,a)}\Vert^2_{\vmu'} + \frac{4}{(1-\gamma)^2}(\Vert \evu_t(s)\ve^{(s,a)}\Vert^2_{\vmu'} + \Vert \gamma \evu_t(s')\ve^{(s,a)}\Vert^2_{\vmu'})\right] \\
        &=  |\St|^2|\Ac|^2\cdot \E \left[ \evmu'(s,a)\left(2(\evc_t(s,a))^2 + \frac{4}{(1-\gamma)^2}(\evu_t(s))^2 + \frac{4\gamma^2}{(1-\gamma)^2}(\evu_t(s'))^2\right) \right] \\
        &= |\St||\Ac|\Bigg[\sum\limits_{s,a}\evmu'(s,a)\left(2(\evc_t(s,a))^2 + \frac{4}{(1-\gamma)^2}(\evu_t(s))^2\right)\\
        &\quad + \sum\limits_{s',s,a}\evmu'(s,a)P(s'\mid s,a)\frac{4\gamma^2}{(1-\gamma)^2}(\evu_t(s'))^2 \Bigg]\\
        &\stackrel{(ii)}{\leq} |\St||\Ac|\left[\left(2 + \frac{4}{(1-\gamma)^2}\right) \sum\limits_{s,a}\evmu'(s,a) +\frac{4\gamma^2}{(1-\gamma)^2}\sum\limits_{s',s,a}\evmu'(s,a)P(s'\mid s,a)\right]\\
        &= |\St||\Ac|\left(2 + \frac{4(1+\gamma^2)}{(1-\gamma)^2}\right),
    \end{align*}
    where we used $\Vert \vx+ \vy\Vert^2 \leq 2[\Vert \vx \Vert^2 + \Vert \vy \Vert^2]$ two times for $(i)$ and that $(\vc_t,\vu_t) \in \sB^{|\St||\Ac|}_{1}\times \sB^{|\St|}_{1} $ for $(ii)$. 
\end{proof}

\subsection{Algorithm convergence}
We will follow \cite{Jin2020} and show how their results accommodate our problem.
\begin{definition}[$\ell_\infty$-$\ell_1$ convex-concave min-max problem] 
    Let \(f: \R^n \times \R^m \to \R\) be a differentiable function that is convex in \(\vx\in\R^n\) and concave in \(\vy\in\R^m\). We define the \(\ell_\infty\)-\(\ell_1\) convex-concave min-max problem as
\begin{equation}\label{eq:convconc}
    \min_{\vx \in \sB^n_b} \max_{\vy \in \Delta^m} \; f(\vx, \vy).
\end{equation}
\end{definition}

Furthermore, define the operator $$G(\vz)=G(\vx,\vy)=[\nabla_\vx f(\vx,\vy), - \nabla_\vy f(\vx,\vy)]= [g^\vx(\vx,\vy), g^\vy(\vx,\vy)].$$

\begin{lemma}[cf. Appendix A.1 in \cite{Carmon2019}] \label{lemma:Gap}
For every $\vz_1,...,\vz_K \in \mathcal{Z} = \sB^n_b \times \Delta^m$ it holds that
\begin{align*}
    \text{Gap}\left(\frac{1}{K}\sum\limits_{k=1}^K \vz_k\right) \leq \sup_{\vu\in\mathcal{Z}} \frac{1}{K}\sum\limits_{k=1}^K \langle G(\vz_k), \vz_k - \vu\rangle.
\end{align*}
\end{lemma}
\begin{proof}
    For all $\vz\in\mathcal{Z}$, $f(\vz^\vx, \vu^\vy)$ is concave in $\vu^\vy$ and $-f(\vu^\vx,\vz^\vy)$ is concave in $\vu^\vx$. Therefore, $\text{gap}(\vz,\vu)$ is concave in $\vu$ for every $\vz$ and we have
    \begin{align*}
        \text{gap}(\vz,\vu) &\leq \text{gap}(\vz,\vz) + \langle \nabla_\vu\text{gap}(\vz,\vz), \vu-\vz \rangle\\
        &= \langle \nabla_\vu\text{gap}(\vz,\vz), \vu-\vz \rangle\\
        &= \langle -G(\vz), \vu-\vz \rangle\\
        &= \langle G(\vz), \vz-\vu \rangle.
    \end{align*}
    Similarly, $\text{gap}(\vz;\vu)$ is convex in $\vz$ for every $\vu$. Therefore,
    \begin{align*}
        \text{gap}\left(\frac{1}{K}\sum\limits_{k=1}^K \vz_k;\vu \right) \leq \frac{1}{K}\sum\limits_{k=1}^K\text{gap}(\vz_k;\vu) \leq \frac{1}{K}\sum\limits_{k=1}^K \langle G(\vz_k), \vz_k - \vu \rangle,
    \end{align*}
    where the first inequality follows from convexity in $\vz$ and the second inequality from the result above. Taking the supremum over the inequality yields the result.
\end{proof}

The two divergences that we will use in the stochastic mirror descent algorithm for the $\ell_\infty$-$\ell_1$ convex-concave min-max problem are the following:
\begin{enumerate}
    \item given the euclidean distance $h(\vx) = \frac{1}{2}\Vert \vx \Vert_2^2$, we obtain the divergence $V_\vx(\vx') = \frac{1}{2}\Vert\vx -\vx'\Vert^2_2$;
    \item given $h(\vy) = \sum_{i}\evy_i\log \evy_i$, we obtain the Kullback-Leibler divergence $V_\vy(\vy') = \sum \evy_i \log\left(\frac{\evy_i'}{\evy_i}\right)$.
\end{enumerate}
\begin{theorem}[c.f. Theorem 3 \citep{Jin2020}]\label{theorem:SMDconvergence}
    Given a $\ell_\infty$-$\ell_1$ convex-concave-min-max problem \eqref{eq:convconc},  desired accuracy $\epsilon$, $(v^\vx, \Vert \cdot \Vert_2)$-bounded estimators $\Tilde{g}^\vx$ of $g^\vx$, and $(\frac{2v^\vy}{\epsilon}, v^\vy, \Vert \cdot \Vert_{\Delta^m})$-bounded estimators $\Tilde{g}^\vy$ of $g^\vy$. Algorithm \ref{alg:our_alg} with choice of parameters $\eta_\vx \leq \frac{\epsilon}{4v^\vx}$, $\eta_\vy \leq \frac{\epsilon}{4v^\vy}$ outputs an expected $\epsilon$-approximate optimal solution within any iteration number $T\geq \max\{\frac{16nb^2}{\epsilon\eta_x}, \frac{8\log(m)}{\epsilon\eta_y}\}$.
\end{theorem}
\begin{proof}
    Note that $V_\vx$ is 1-strongly convex. Since $\eta_\vy\leq \frac{\epsilon}{4v^\vy}$, we have that
    \begin{align*}
        \Vert \eta^\vy \Tilde{g}^\vy_t \Vert_\infty &\leq \frac{\epsilon}{4v^\vy} \cdot \Vert \Tilde{g}^\vy_t \Vert_\infty \leq \frac{\epsilon}{4v^\vy} \cdot \frac{2v^\vy}{\epsilon} = \frac{1}{2}.
    \end{align*}
    Hence, by Lemma 1 and Lemma 2 in \cite{Jin2020} we know that
    \begin{align*}
        \sum\limits_{t\in[T]} \langle \eta^\vx \Tilde{g}^\vx_t, \vx_t - \vx \rangle &\leq V_{\vx_1}(\vx) + \frac{{\eta^\vx}^2}{2}\sum\limits_{t\in[T]}\Vert \Tilde{g}^\vx_t\Vert^2_2,\\
        \sum\limits_{t\in[T]} \langle \eta^\vy \Tilde{g}^\vy_t, \vy_t - \vy \rangle &\leq V_{\vy_1}(\vy) + \frac{{\eta^\vy}^2}{2}\sum\limits_{t\in[T]}\Vert \Tilde{g}^\vy_t\Vert^2_{\vy_t}.
    \end{align*}
    Now, define $\hat{g}^\vx_t := g^\vx_t - \Tilde{g}^\vx_t$, $\hat{g}^\vy_t := g^\vy_t - \Tilde{g}^\vy_t$, and the sequences $\hat{\vx}_1,...,\hat{\vx}_T$ and $\hat{\vy}_1,...,\hat{\vy}_T$ by
    \begin{align*}
        \hat{\vx}_1 = \vx_1, \: \hat{\vx}_{t+1} = \argmin_{\vx\in\sB^n_b} \langle \eta^\vx\hat{g}^\vx_t, \vx \rangle + V_{\hat{\vx}_t}(\vx),\\
        \hat{\vy}_1 = \vy_1, \: \hat{\vy}_{t+1} = \argmin_{\vy\in\Delta^m} \langle \eta^\vy\hat{g}^\vy_t, \vy \rangle + V_{\hat{\vy}_t}(\vy).
    \end{align*}
    In a similar way to $\eta^\vy g^\vy_t$, we can bound the $\ell_\infty$-norm of $\eta^\vy \hat{g}^\vy_t$
    \begin{align*}
        \Vert \eta^\vy \hat{g}^\vy_t\Vert_\infty \leq \Vert \eta^\vy \Tilde{g}^\vy_t\Vert_\infty + \Vert \eta^\vy g^\vy_t\Vert_\infty = \Vert \eta^\vy \Tilde{g}^\vy_t\Vert_\infty + \Vert \E[\eta^\vy \Tilde{g}^\vy_t]\Vert_\infty \leq 2\Vert \eta^\vy \Tilde{g}^\vy_t\Vert_\infty \leq 1.
    \end{align*}
    Therefore, using the lemmas as above we get
    \begin{align*}
        \sum\limits_{t\in[T]} \langle \eta^\vx \hat{g}^\vx_t, \vx_t - \vx \rangle &\leq V_{\vx_1}(\vx) + \frac{{\eta^\vx}^2}{2}\sum\limits_{t\in[T]}\Vert \hat{g}^\vx_t\Vert^2_2,\\
        \sum\limits_{t\in[T]} \langle \eta^\vy \hat{g}^\vy_t, \vy_t - \vy \rangle &\leq V_{\vy_1}(\vy) + \frac{{\eta^\vy}^2}{2}\sum\limits_{t\in[T]}\Vert \hat{g}^\vy_t\Vert^2_{\vy_t}.
    \end{align*}
    Since $ g^\vx_t = \hat{g}^\vx_t + \Tilde{g}^\vx_t$ and $ g^\vy_t = \hat{g}^\vy_t + \Tilde{g}^\vy_t$, 
    \begin{align*}
        &\sum\limits_{t\in[T]} \left[ \langle g^\vx_t, \vx_t - \vx\rangle + \langle g^\vy_t, \vy_t -\vy \rangle \right]\\
        = &\sum\limits_{t\in[T]}\left[ \frac{1}{\eta^\vx}\langle \eta^\vx\Tilde{g}^\vx_t, \vx_t - \vx\rangle + \frac{1}{\eta^\vy}\langle \eta^\vy\Tilde{g}^\vy_t, \vy_t - \vy\rangle\right] + \sum\limits_{t\in[T]}\left[ \frac{1}{\eta^\vx}\langle \eta^\vx\hat{g}^\vx_t, \hat{\vx}_t - \vx\rangle + \frac{1}{\eta^\vy}\langle \eta^\vy\hat{g}^\vy_t, \hat{\vy}_t - \vy\rangle \right]\\
        + &\sum\limits_{t\in[T]}\left[\langle \hat{g}^\vx_t, \vx_t - \hat{\vx}_t \rangle + \langle \hat{g}^\vy_t, \vy_t - \hat{\vy}_t \rangle\right]\\
        = &\frac{1}{\eta^\vx}\sum\limits_{t\in[T]}\left[\langle \eta^\vx\Tilde{g}^\vx_t, \vx_t - \vx\rangle \right] + \frac{1}{\eta^\vx}\sum\limits_{t\in[T]}\left[ \langle \eta^\vx\hat{g}^\vx_t, \hat{\vx}_t - \vx\rangle \right] + \frac{1}{\eta^\vy}\sum\limits_{t\in[T]}\left[\langle \eta^\vy\Tilde{g}^\vy_t, \vy_t - \vy\rangle\right] \\
+& \frac{1}{\eta^\vy} \sum\limits_{t\in[T]} \left[ \langle \eta^\vy\hat{g}^\vy_t, \hat{\vy}_t - \vy\rangle \right]
        + \sum\limits_{t\in[T]}\left[\langle \hat{g}^\vx_t, \vx_t - \hat{\vx}_t \rangle + \langle \hat{g}^\vy_t, \vy_t - \hat{\vy}_t \rangle\right]\\
        \leq &\frac{2}{\eta^\vx}V_{\vx_1}(\vx) + \frac{{\eta^\vx}}{2}\sum\limits_{t\in[T]} \left[ \Vert\Tilde{g}^\vx_t\Vert^2_2 + \Vert\hat{g}^\vx_t\Vert^2_2 \right] + \sum\limits_{t\in[T]}\langle \hat{g}^\vx_t, \vx_t - \hat{\vx}_t \rangle\\
        + &\frac{2}{\eta^\vy}V_{\vy_1}(\vy) + \frac{{\eta^\vy}}{2}\sum\limits_{t\in[T]} \left[ \Vert\Tilde{g}^\vy_t\Vert^2_{\vy_t} + \Vert\hat{g}^\vy_t\Vert^2_{\vy_t} \right] + \sum\limits_{t\in[T]}\langle \hat{g}^\vy_t, \vy_t - \hat{\vy}_t \rangle.
        \end{align*}
        Consider the operator $G(\vz) := [g^\vx(\vx,\vy), -g^\vy(\vx,\vy)]$. As min-max problem \ref{eq:convconc} is convex in its first argument and concave in its second argument, by Lemma \ref{lemma:Gap}, if we show that 
        \begin{align*}
            \sup_{\vu\in\mathcal{Z}} \frac{1}{T}\sum\limits_{t\in[T]} \langle G(\vz_t), \vz_t - \vu\rangle \leq \epsilon,
        \end{align*}
        we obtain that $\text{Gap}(\frac{1}{T}\sum\limits_{t=1}^T \vz_t) \leq \epsilon$. Now take the expectation on both sides: 
        \begin{align*}
            \E\sup_{\vu\in\mathcal{Z}} \frac{1}{T}\sum\limits_{t\in[T]} \langle G(\vz_t), \vz_t - \vu\rangle &=  \E \frac{1}{T} \sup_{(\vx,\vy)}\left[ \sum\limits_{t\in[T]} \langle g^\vx(\vx,\vy), \vx_t - \vx \rangle + \sum\limits_{t\in[T]} \langle  g^\vy(\vx,\vy), \vy_t - \vy \rangle \right]\\
            &\stackrel{(i)}{\leq} \frac{1}{T} \E \sup_{(\vx,\vy)} \Bigg[ \frac{2}{\eta^\vx}V_{\vx_1}(\vx) + \frac{{\eta^\vx}}{2}\sum\limits_{t\in[T]} \left[ \Vert\Tilde{g}^\vx_t\Vert^2_2 + \Vert\hat{g}^\vx_t\Vert^2_2 \right]\\
            &+ \frac{2}{\eta^\vy}V_{\vy_1}(\vy) + \frac{{\eta^\vy}}{2}\sum\limits_{t\in[T]} \left[ \Vert\Tilde{g}^\vy_t\Vert^2_2 + \Vert\hat{g}^\vy_t\Vert^2_2 \right] \Bigg]\\
            &\stackrel{(ii)}{\leq} \frac{1}{T} \E \sup_{(\vx,\vy)} \Bigg[ \frac{2}{\eta^\vx}V_{\vx_1}(\vx) + \eta^\vx \sum\limits_{t\in[T]}\Vert\Tilde{g}^\vx_t\Vert^2_2 + \frac{2}{\eta^\vy}V_{\vy_1}(\vy) + \eta^\vy \sum\limits_{t\in[T]} \Vert\Tilde{g}^\vy_t\Vert^2_2 \Bigg]\\
            &\stackrel{(iii)}{\leq}  \sup_{\vx}\frac{2}{\eta^\vx T}V_{\vx_1}(\vx) + \eta^\vx v^\vx + \sup_{\vy}\frac{2}{\eta^\vy T}V_{\vy_1}(\vy) + \eta^\vy v^\vy\\
            &\stackrel{(iv)}{\leq}  \frac{4nb^2}{\eta^\vx T} + \eta^\vx v^\vx + \frac{2\log m}{\eta^\vy T} + \eta^\vy v^\vy\\
            &\stackrel{(v)}{\leq}  \epsilon,
        \end{align*}
        where in $(i)$ we used that $\E[\langle \hat{g}^\vx_t, \vx_t - \hat{\vx}_t \rangle \: | \: 1,...,T] = \E[\langle \hat{g}^\vy_t, \vy_t - \hat{\vy}_t \rangle \: | \: 1,...,T] = 0$; $(ii)$ $\E[\Vert \hat{g}^\vx_t \Vert^2_2] \leq \E[\Vert \Tilde{g}^\vx_t \Vert^2_2]$ and $\E[ \sum_i [\hat{\evy}_t]_i  [\hat{g}^\evy_t]^2_i] \leq \E[ \sum_i [\hat{\evy}_t]_i  [\Tilde{g}^\evy_t]^2_i]$ due to $\E[(X-\E[X])^2] \leq \E[X^2]$; $(iii)$ due to the assumptions on the estimators; $(iv)$ by properties of KL-divergence and that $\frac{1}{2}\Vert\vx-\vx_0\Vert^2_2 \leq 2nb^2$; $(v)$ the choice of $\eta^\vx = \frac{\epsilon}{4v^\vx}$, $\eta^\vy = \frac{\epsilon}{4v^\vy}$, and $T \geq \max\{\frac{16nb^2}{\epsilon\eta^\vx}, \frac{8\log m}{\epsilon\eta^\vy}\}$.
\end{proof}

Now, Theorem \ref{theorem:algconvergence}  follows directly from the bounds of the gradient estimators and Theorem \ref{theorem:SMDconvergence}.

\section{Numerical experiments}

\subsection{Gridworld}
We extend the analysis to Gridworld environments with four random obstacle-goal configurations to assess the method's applicability to arbitrary cost structures. We retain the experimental setup described in Section \ref{sec:gridworld}, fixing the regularization parameter to $\alpha=0.005$ and generating a suboptimal expert via early solver termination. For each configuration, we compare the resulting apprentice policy against both the optimal and expert policies, and examine the learned cost vector for the action ``down" relative to the ground truth given by the color coding of the policies' grid.

\begin{figure}[ht]
    \centering
    % Use our new 'M' column type. 
    % I slightly reduced 0.2 to 0.19 to ensure it fits within page margins with padding.
    \begin{tabular}{ M{0.75cm} M{0.19\textwidth} M{0.19\textwidth} M{0.19\textwidth} M{0.19\textwidth} }
    
         & Optimal policy & Suboptimal expert & Learned cost (down) & Apprentice policy \\
         
         \text{Grid 1} 
            & \includegraphics[width=\linewidth]{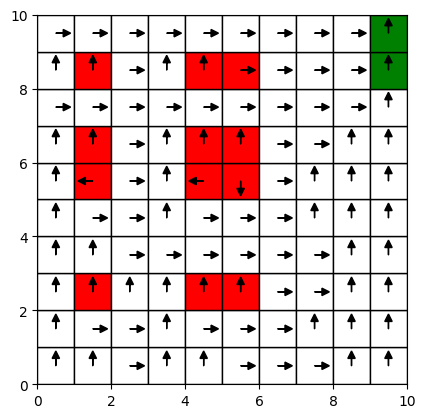}
            & \includegraphics[width=\linewidth]{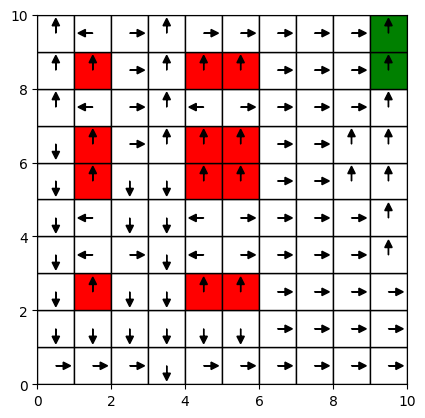}
            & \includegraphics[width=\linewidth]{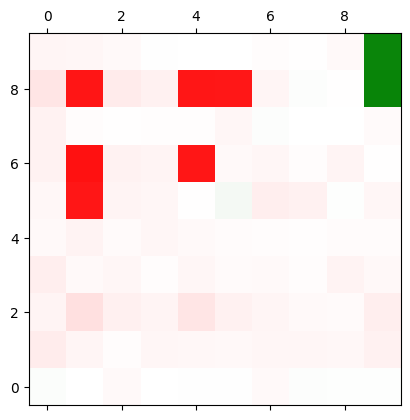}
            & \includegraphics[width=\linewidth]{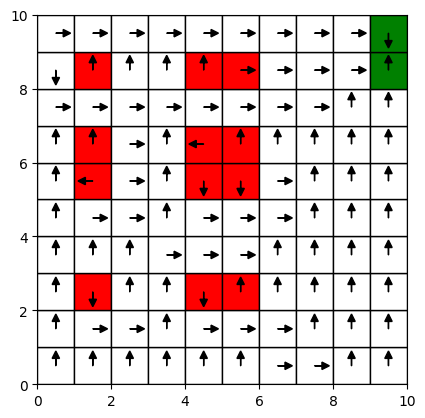} \\
            
         \text{Grid 2}
            & \includegraphics[width=\linewidth]{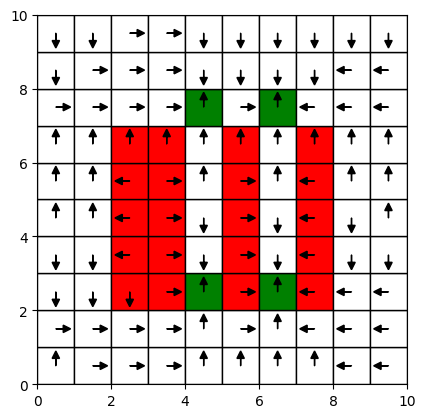}
            & \includegraphics[width=\linewidth]{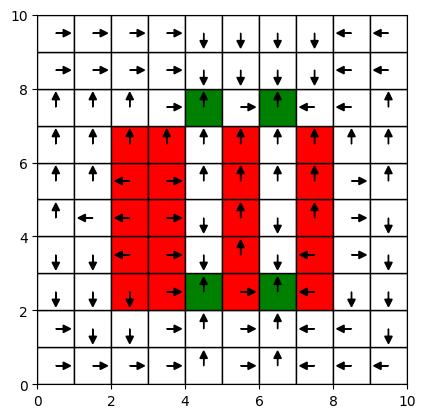}
            & \includegraphics[width=\linewidth]{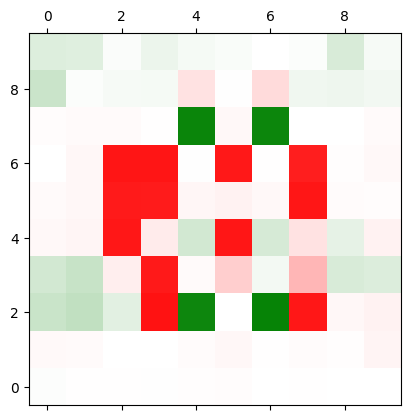}
            & \includegraphics[width=\linewidth]{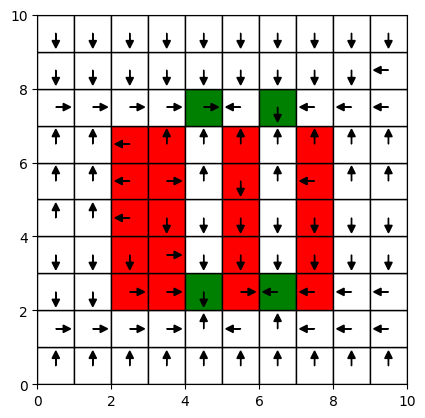}\\
            
        \text{Grid 3}
            & \includegraphics[width=\linewidth]{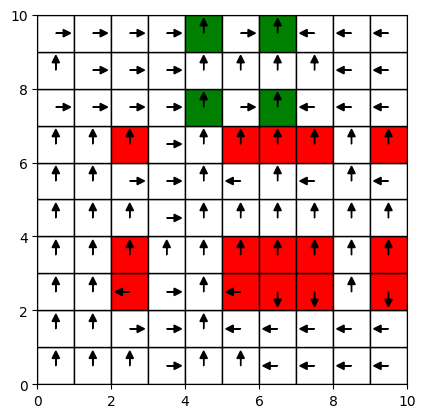}
            & \includegraphics[width=\linewidth]{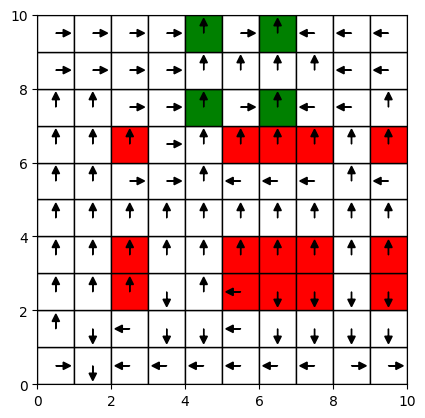}
            & \includegraphics[width=\linewidth]{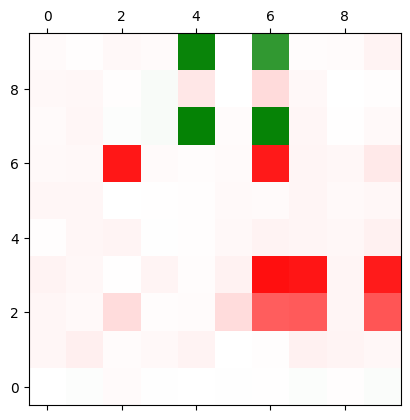}
            & \includegraphics[width=\linewidth]{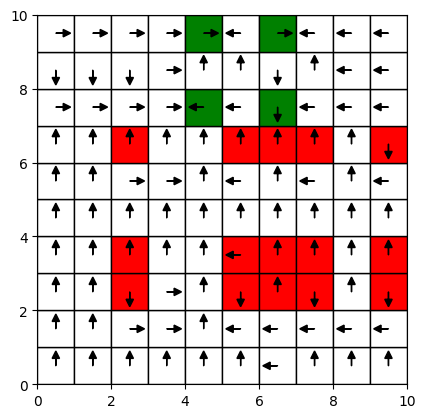} \\
            
        \text{Grid 4}
            & \includegraphics[width=\linewidth]{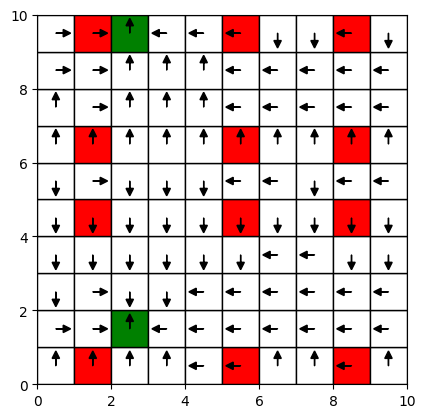}
            & \includegraphics[width=\linewidth]{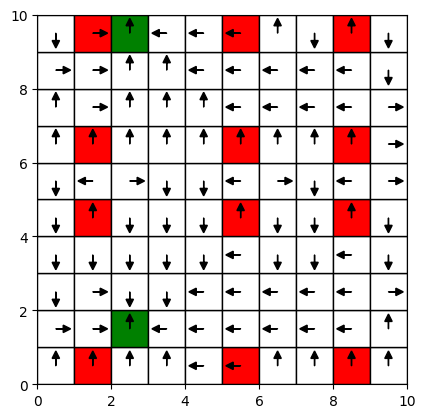}
            & \includegraphics[width=\linewidth]{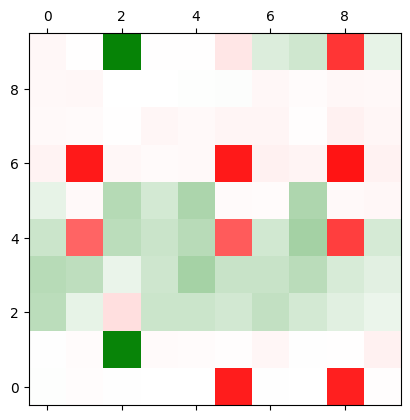}
            & \includegraphics[width=\linewidth]{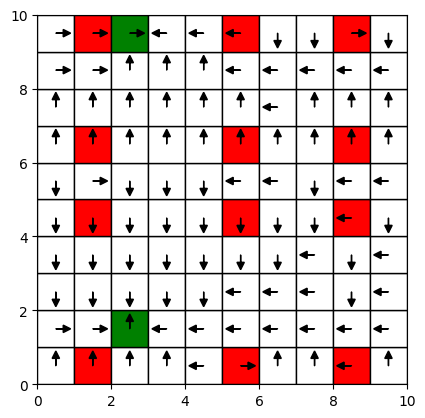}
            
    \end{tabular}
    \caption{Effect of the regularization on the cost vector.}
    \label{fig:gridappendix}
\end{figure}

\end{document}